\documentclass[11pt]{article}
\usepackage{amsmath,amssymb,amsthm,bm}
\usepackage[margin=1in]{geometry}
\usepackage{enumitem}
\usepackage{hyperref}
\usepackage{xcolor}
\usepackage{courier}
\usepackage{pgfplots}
\pgfplotsset{compat=1.18}
\usepackage{booktabs}

\definecolor{cblue}{HTML}{1f77b4}
\definecolor{cred}{HTML}{d62728}
\usepackage{xcolor}
\usepackage{pgfplots}
\pgfplotsset{compat=1.18}

\definecolor{cgray}{HTML}{888888}

\newcommand{\X}{\mathcal{X}}

\newcommand{\E}{\mathbb{E}}
\newcommand{\Prob}{\mathbb{P}}

\newcommand{\dkl}[2]{D_{\mathrm{KL}}\!\left(#1\,\|\,#2\right)}
\newcommand{\tv}[1]{\mathrm{TV}\!\left(#1\right)}

\usepackage{amsfonts}
\usepackage{times}
\usepackage{latexsym}
\usepackage{amssymb}
\usepackage{amsmath}
\usepackage{verbatim}
\newtheorem{theorem}{Theorem}
\newtheorem{assumption}[theorem]{Assumption}

\newtheorem{remark}{Remark}
\newtheorem{definition}[theorem]{Definition}
\newtheorem{lemma}[theorem]{Lemma}

\def\bb0{{\mathbb{0}}}

\def\bb{{\mathbf{b}}}

\def\b0{{\mathbf{0}}}

\def\b1{{\mathbf{1}}}

\def\bbP{{\mathbb{P}}}

\def\cF{\mathcal{F}}

\def\cP{\mathcal{P}}

\def\sf0{{\mathsf{0}}}

\title{Simultaneous Coverage and Efficiency Guarantee in  \\ Online Conformal Prediction}
\author{Rahul Vaze \\ School of Technology and Computer Science \\ Tata Institute of Fundamental Research, Mumbai\\ \texttt{\small rahul.vaze@gmail.com}}
\date{}

\begin{document}
\maketitle

\begin{abstract}
Adaptive conformal inference (ACI) of Gibbs and Cand{\`e}s \cite{gibbs2021adaptive} and its variants are the standard
approach to online conformal prediction under distribution shift, but they
suffer from three fundamental limitations. First, their guarantees control
only the \emph{signed} long-run coverage error: persistent miscoverage in
one direction can be masked by compensating errors later, so a method can
satisfy the theoretical guarantee while being badly wrong for extended
periods. Second, existing guarantees say nothing about prediction-set
size, so validity can be achieved trivially at the cost of unduly wide
prediction sets. Third, the efficiency guarantees that do exist compare
against a \emph{fixed} predictor chosen in hindsight, a benchmark that
becomes increasingly less meaningful once the data-generating distribution shifts,
since the very notion of an optimal threshold then changes over time.

We consider a unified online learning framework that simultaneously controls absolute, non-cancelling coverage violation and prediction-set efficiency against a dynamically evolving benchmark for three important models. In the fully adversarial setting, exploiting the fact that the standard ACI update  is exactly projected online gradient descent on the pinball loss, we derive  simultaneous coverage and efficiency guarantees for arbitrary monotone Lipschitz efficiency objectives, with no distributional or {\it convexity} assumptions. In the stochastic setting with full-score feedback, we propose a sliding-window quantile tracker and establish a matching minimax lower bound  showing our algorithm is rate-optimal. In the covariate-dependent stochastic setting, we develop a partitioned ACI algorithm that tracks a function-valued oracle threshold, and derive simultaneous coverage and efficiency guarantees. 

Together, these results give the first framework offering simultaneous, non-cancelling coverage and efficiency guarantees for online conformal prediction under non-stationarity, and precisely characterize how performance degrades as feedback becomes coarser and the target becomes covariate-dependent.\end{abstract}

\section{Introduction}
\label{sec:background}

Machine learning models are increasingly used in applications where uncertainty quantification is as important as accurate prediction. Rather than producing a single prediction, one seeks to construct a prediction set that contains the unknown response with high probability. Conformal prediction (CP) \cite{vovk1999machine,vovk2005algorithmic,shafer2008tutorial,angelopoulos2023gentle} provides a general framework for doing so, where  given a feature vector $X$, one constructs a prediction set
\[
C(X)\subseteq\mathcal Y\footnote{Set from which the response variable $Y$ takes values.},
\]
for the corresponding response variable $Y$. The objective is to construct prediction sets satisfying
\begin{equation}\label{defn:coveragecond}
\bbP\!\left(Y\in C(X)\right)\ge 1-\alpha,
\end{equation}
where $\alpha\in(0,1)$ is the target miscoverage level
\cite{vovk2005algorithmic,lei2018distribution}. Condition \eqref{defn:coveragecond} is defined as {\it coverage}.

CP constructs the prediction set using a \emph{conformity score}
\[
s:\mathcal X\times\mathcal Y\rightarrow\mathbb R,
\]
which measures the compatibility of a candidate response with the observed feature vector. Given a threshold $q$, the prediction set is defined as
\[
C(X)
=
\{y\in\mathcal Y:s(X,y)\le q\}.
\]
For regression, $C(X)$ is typically an interval, whereas for classification it is a subset of labels. A distinguishing feature of CP is that the conformity score can be constructed from virtually any underlying prediction model, making the framework largely model agnostic
\cite{lei2018distribution,romano2019conformalized,angelopoulos2023gentle}.

\subsection{Prediction Set Efficiency}

Coverage alone does not characterize the quality of a prediction set. For example, the trivial prediction set
\[
C(X)=\mathcal Y
\]
achieves perfect coverage but provides essentially no predictive information. Among all prediction sets achieving the desired coverage level, one generally prefers those that are as small as possible. Consequently, prediction efficiency has become a central objective in modern conformal prediction
\cite{romano2019conformalized,angelopoulos2021raps,gao2025volume}.
For regression, efficiency is commonly measured by the average prediction interval length,
\[
\frac1T\sum_{t=1}^T |C_t|,
\]
while for general prediction sets it is measured by their volume, cardinality, or an analogous notion of size
\cite{angelopoulos2023gentle,gao2025volume}.

Measuring efficiency directly through the geometry of the prediction set is intrinsic to the prediction problem and is independent of the particular parameterization used to construct the prediction set. This viewpoint has recently motivated optimization-based formulations of conformal prediction that explicitly minimize prediction-set volume subject to coverage constraints
\cite{gao2025volume}.

\subsection{Online Conformal Prediction}

Classical conformal prediction is fundamentally a \emph{static} inference problem. A prediction rule is calibrated once using an {\it exchangeable}\footnote{Formally, a sequence of random variables $Z_1, \dots, Z_n$ is exchangeable if their joint distribution is invariant under any permutation of their indices; i.e., $(Z_1, \dots, Z_n) \overset{d}{=} (Z_{\pi(1)}, \dots, Z_{\pi(n)})$ for any permutation $\pi$ of $\{1, \dots, n\}$. This is a strictly weaker assumption than being independent and identically distributed (i.i.d.).} calibration dataset and is then applied to future observations generated from the same underlying distribution.  Consequently, a single calibration threshold is sufficient to provide finite-sample coverage guarantees throughout the prediction process.

Many practical applications, however, violate this assumption. In streaming data analysis, autonomous systems, finance, healthcare, recommendation systems, and time-series forecasting, the data-generating mechanism evolves continuously over time. Rather than observing samples from a fixed distribution,
\[
(X_t,Y_t)\sim P,
\]
one instead observes a sequence
\[
(X_t,Y_t)\sim P_t,\qquad t=1,2,\ldots,
\]
where the distributions $\{P_t\}$ are allowed to change over time. As the underlying distribution drifts, a threshold calibrated using historical data gradually becomes obsolete, causing the empirical coverage of the resulting prediction sets to deteriorate.

Online conformal prediction studies precisely this sequential setting. At each round $t$, the learner observes the current feature vector $X_t$, constructs a prediction set $C_t(X_t)$ using only the information available up to time $t-1$, subsequently observes the true response $Y_t$, and updates its prediction rule before the next round. Given a sequence of observations generated from the time-varying distributions $\{P_t\}$, the objective is to construct prediction sets
\[
C_1,C_2,\ldots,C_T,
\]
whose empirical coverage satisfies
\[
\frac1T\sum_{t=1}^{T}
\mathbf1\{Y_t\in C_t(X_t)\}
\approx
1-\alpha,
\]
despite the underlying distribution changing over time.

Unlike classical conformal prediction, in the considered online setting, the learner must continually adapt its prediction rule while making predictions as the underlying distribution evolves. This creates a fundamental stability--adaptivity tradeoff. Updating too aggressively enables the prediction sets to respond rapidly to distribution shifts but increases sensitivity to noisy observations, whereas updating too conservatively produces more stable prediction sets that may fail to track changes in the underlying data-generating process. 

The modern study of online conformal prediction was initiated by Gibbs and Cand\`es through Adaptive Conformal Inference (ACI) \cite{gibbs2021adaptive}. Let
\[
z_t=\mathbf1\{Y_t\in C_t(X_t)\}
\]
denote the coverage indicator. Starting from an initial threshold $q_1$, ACI updates the conformity threshold according to
\[
q_{t+1}
=
\Pi_{\mathcal Q}
\!\left(
q_t+\eta\bigl((1-z_t)-\alpha\bigr)
\right),
\]
where $\eta>0$ is the learning rate and $\Pi_{\mathcal Q}$ denotes projection onto the admissible threshold set. Thus, whenever the prediction set fails to cover the response ($z_t=0$), the threshold is increased, whereas successful coverage ($z_t=1$) decreases the threshold, allowing the prediction sets to continually adapt over time.

Subsequent work has substantially expanded this framework through adaptive learning rates, multiscale adaptation, expert aggregation, online optimization techniques, and improved finite-time analyses
\cite{zaffran2022adaptive,bhatnagar2023improved,bhatnagar2024online,feldman2023online}. Related methods have also been developed for time-series forecasting, dependent observations, and other forms of distribution shift
\cite{xu2021enbpi,stankeviciute2021conformal,podkopaev2021distribution,barber2023conformal}. Despite their methodological differences, these approaches share a common objective: maintaining coverage under evolving data distributions, with theoretical guarantees typically expressed in terms of marginal coverage, empirical coverage, calibration error, or long-run coverage.

Recent work has connected online conformal prediction with online optimization. Areces et al.~\cite{pmlr-v267-areces25a} formulate online conformal prediction as an online optimization problem and establish coverage guarantees in adversarial and stochastic settings, while Ramalingam et al.~\cite{pmlr-v267-ramalingam25a} characterize the relationship between no-regret learning and online conformal prediction, showing how regret guarantees imply calibration and conditional coverage. 


\subsection{Fundamental Limitations}

\paragraph{Joint Coverage and Efficiency Guarantees.}

Conformal prediction has two natural objectives: coverage and prediction-set efficiency. 
When the underlying distribution remains unchanged, i.e., $P_t\equiv P$, the problem of simultaneously guaranteeing coverage while improving prediction-set efficiency has been studied extensively. Early works by Lei \emph{et al.}~\cite{lei2013distribution} and Sadinle \emph{et al.}~\cite{sadinle2019least} investigated minimum-volume prediction regions under structural assumptions. Subsequently, Izbicki, Shimizu, and Stern~\cite{izbicki2020flexible,izbicki2022conformal} developed conformal prediction procedures that produce shorter prediction intervals by leveraging increasingly accurate estimation of the conditional distribution while preserving finite-sample marginal coverage. Kiyani \emph{et al.}~\cite{kiyani2024} further studied prediction-length optimization in conformal regression. More recently, Gao \emph{et al.}~\cite{gao2025volume} established that unrestricted distribution-free volume optimality is impossible in general, and proposed a notion of restricted volume optimality over structured families of prediction sets, for which finite-sample guarantees can be obtained. 
Collectively, these works establish a rich theory of prediction-set efficiency for the classical exchangeable setting, but do not address sequential prediction under distribution shift.

Another line of work related to online conformal prediction was recently introduced by Srinivas~\cite{srinivas2025}, who studied an online interval prediction problem under arbitrary input sequences. At each round, the learner outputs a prediction interval before observing the response, with the objective of simultaneously maintaining long-run coverage and minimizing cumulative interval length relative to the best fixed interval satisfying the desired coverage level in hindsight. This work establishes fundamental tradeoffs between coverage and prediction efficiency in an adversarial online setting. However, it considers an abstract online interval prediction problem rather than adaptive conformal prediction based on conformity scores, and measures performance relative to a \emph{static} comparator, namely the single fixed prediction interval that minimizes cumulative interval length while satisfying the desired long-run coverage constraint in hindsight.


However, as far as we know, joint coverage and efficiency guarantees for online conformal prediction under distribution shift remains unavailable.

\paragraph{Signed versus feasibility-based coverage error.}

Existing adaptive conformal prediction methods quantify calibration through the signed coverage error
\[
\sum_{t=1}^{T}
\left(
\mathbf1\{Y_t\notin C_t(X_t)\}
-
\alpha
\right),
\]
or closely related notions of calibration error
\cite{gibbs2021adaptive}. Since positive and negative deviations cancel, prolonged periods of under-coverage may be offset by later periods of over-coverage, even though the coverage requirement has been violated during the earlier rounds. While this is a natural metric for calibration, it does not distinguish between satisfying the coverage constraint and violating it.


\paragraph{Static versus Dynamic Benchmarks.}

The final challenge concerns the choice of performance benchmark in online non-stationary environments. Existing optimization-based formulations compare the online algorithm with the best fixed predictor chosen in hindsight
\cite{gao2025volume,srinivas2025}. Such static benchmarks are natural when the underlying distribution remains unchanged, but become increasingly restrictive as the data-generating distribution evolves over time. In particular, if
\[
(X_t,Y_t)\sim P_t,
\]
the prediction threshold that simultaneously achieves the desired coverage level and minimizes prediction-set size generally depends on the current distribution $P_t$, and therefore cannot, in general, be represented by a single fixed predictor.


\subsection{Problem Formulation and Results}
To address the above mentioned limitations, we consider the following general online conformal prediction problem. At each round $t$, after observing the feature vector $X_t$, the learner selects a conformity threshold $q_t$ and constructs the prediction set $C_t(X_t)
=
\{y\in\mathcal Y:s(X_t,y)\le q_t\}.$
After the true response $Y_t$ is revealed, the realized conformity score $r_t=s(X_t,Y_t)$
becomes available.
The prediction set naturally induces an efficiency objective through its geometry. Accordingly, we define
$\Lambda(q_t)=\mu\!\left(C_t(X_t)\right)
=\mu(q,x):=\lambda(\{y:s(x,y)\le q\}),$
where $\mu(\cdot)$ denotes an intrinsic measure of prediction-set size, such as interval length, prediction-set volume, or cardinality. 
\begin{remark}[Efficiency cost need not be convex]
\label{rem:nonconvex}
The set-size function $\mu(q,x)$ is always
nondecreasing in $q$ (larger threshold $\Rightarrow$ larger set) but is
generically \emph{non-convex}. Our results require
only monotonicity and Lipschitz continuity. Thus $\Lambda$ is assumed to be monotonically non-decreasing and $L$-Lipschitz.
\end{remark}

The resulting online optimization problem is to minimize cumulative prediction-set size while maintaining the desired coverage constraint. To evaluate performance, we compare the learner against an optimal dynamic benchmark appropriate for the statistical model under consideration. Let
\begin{equation}\label{defn:optbenchmark}
\mathbf u^*
=
(u_1^*,\ldots,u_T^*)
\end{equation}
denote this optimal comparator sequence. The precise definition of $\mathbf u^*$ depends on the assumptions imposed on the conformity-score process and will be specified separately for each of the three models considered in the sequel.

The efficiency performance of an online algorithm is measured by the {\bf dynamic regret}
\begin{equation}\label{defn:regret}
R(T)
=
\sum_{t=1}^{T}|\Lambda(q_t)
-
\Lambda(u_t^*)|.
\end{equation}
while the coverage performance is measured by a {\bf cumulative coverage violation} functional
\begin{equation}\label{defn:coverage}
Q(T)
=
\sum_{t=1}^{T} c_t(q_t),
\end{equation}
where the nonnegative function $c_t(\cdot)$ quantifies the extent to which the prediction threshold fails to satisfy the desired coverage requirement at round $t$.


The objective is to simultaneously minimize both $R(T)$ and $Q(T)$. The optimization framework is identical throughout the paper; the only distinction between the three models lies in the assumptions imposed on the conformity-score process $\{r_t\}_{t=1}^T$, which determine the appropriate benchmark sequence $\mathbf u^*$ and the corresponding coverage violation function $c_t(\cdot)$.

\subsection{Model I: Adversarial Conformity Scores.}\label{subsec:model1}

The conformity scores
\[
r_1,\ldots,r_T
\]
are generated by an arbitrary adaptive adversary. In this setting, coverage is deterministic: the prediction set covers the response at round $t$ if and only if
\[
q_t\ge r_t.
\]
Accordingly, the optimal dynamic benchmark \eqref{defn:optbenchmark} is the  sequence
\[
\mathbf u^*
\in
\arg\min_{\{u_t\ge r_t\}}
\sum_{t=1}^T \Lambda(u_t),
\]
which, since each $\Lambda$ is monotone nondecreasing, satisfies $u_t^*=r_t$ for every $t$. Thus, dynamic regret \eqref{defn:regret}
\begin{equation}\label{defn:regretwc}
R(T)
=
\sum_{t=1}^{T}
|\Lambda(q_t)
-
\Lambda(r_t)|.
\end{equation}
Let $(.)_+$ denote $\max\{., 0\}$.
With respect to \eqref{defn:coverage}, the coverage violation function is
\[
c_t(q_t)=(r_t-q_t)_+,
\]
and
\begin{equation}\label{defn:coveragewc}
Q(T)=\sum_{t=1}^T(r_t-q_t)_+.
\end{equation}

We show that the classical Adaptive Conformal Inference (ACI) update of Gibbs and Cand\`es can be used to obtain simultaneous guarantees on prediction efficiency and coverage for any monotone Lipschitz (not necessarily convex) efficiency objective:
\[
Q(T),\;R(T)
=
O\!\left(\sqrt{T(1+\cP_T)}\right),
\]
where
\[
{\cal P}_T=\sum_{t=2}^T|r_t-r_{t-1}|
\]
is the path length of the optimal dynamic benchmark sequence. Thus, the standard ACI algorithm enjoys sublinear dynamic regret and sublinear cumulative coverage violation without any stochastic assumptions.

\subsection{Model II: Nonstationary Stochastic Conformity Scores.}\label{subsec:model2}

The conformity scores satisfy
\[
r_t\sim P_t,
\qquad t=1,\ldots,T,
\]
where the distributions
\[
P_1,P_2,\ldots,P_T
\]
are unknown and are allowed to vary arbitrarily over time. Let
\[
F_t(q)=\bbP_{P_t}(r_t\le q)
\]
denote the corresponding cumulative distribution function. In this setting, the desired coverage constraint is
\[
F_t(q_t)\ge 1-\alpha,
\]
so the optimal dynamic benchmark \eqref{defn:optbenchmark} is the oracle quantile sequence
\[
u_t^*
=
F_t^{-1}(1-\alpha).
\]
and dynamic regret is \eqref{defn:regret} with $u_t^*
=
F_t^{-1}(1-\alpha)$. 
The coverage violation function is
\[
c_t(q_t)
=
\mathbb{E}\!\left[\left|F_t(q_t)-(1-\alpha)\right|\right],
\]
so that \eqref{defn:coverage} is 
\begin{equation}\label{defn:coverageprob1}
Q(T)
=
\sum_{t=1}^T
\mathbb{E}\!\left[\left|F_t(q_t)-(1-\alpha)\right|\right].
\end{equation}

For this setting, we propose a sliding-window empirical quantile algorithm and characterize the optimal tradeoff between estimation error and distribution drift. The algorithm achieves
\[
R(T),\;Q(T)
=
O\!\left(T^{2/3}V_T^{1/3}\right),
\]
up to lower-order terms, where
\begin{definition}
\label{def:vt}
\[V_T
=
\sum_{t=1}^{T-1}
\sup_{q\in A}
|F_{t+1}(q)-F_t(q)|\]
denotes the variation budget of the evolving score distributions. 
\end{definition} 

We further establish the matching minimax lower bound
\[
\Omega\!\left(T^{2/3}V_T^{1/3}\right),
\]
showing that this rate is information-theoretically optimal for any online algorithm with access to the full conformity scores.

\paragraph{Relation to non-stationary stochastic optimization.}
Model~II shares the high-level variation-budget framework and the minimax rate $O(T^{2/3}V_T^{1/3})$ established by Besbes, Gur, and Zeevi (BGZ)~\cite{besbes2015nonstationary}. However, Model~II is not a special case of the BGZ framework; the two settings are fundamentally distinguished by their geometric assumptions, feedback structures, and regret formulations. Mathematically, BGZ strictly requires the convexity of the underlying cost functions and action space to guarantee gradient-based convergence. In contrast, Model~II makes no such geometric assumptions, tracking non-parametric quantiles over arbitrary, potentially non-convex conformity-score distributions. Furthermore, while BGZ minimizes unsigned, one-sided regret relative to a dynamic infimum (where the learner cannot outperform the oracle), Model~II must control signed tracking errors, since both positive (over-coverage) and negative (under-coverage) deviations penalize the absolute cumulative coverage violation. Consequently, standard online convex optimization (OCO) batch reductions designed for one-sided convex minimization cannot exploit Model~II's global, full-sample feedback. Instead, Model~II relies on non-convex empirical process theory, specifically, uniform concentration bounds and a quantile-crossing lemma, to simultaneously enforce prediction-set efficiency and absolute coverage, showing that while both models share an estimation--drift tradeoff, their algorithmic and analytical paradigms are distinct.


\subsection{Model III: Nonstationary Covariate-Dependent Conformity Scores.}\label{subsec:model3}
Model II is appropriate when only \emph{marginal} coverage guarantees are required. In many
applications, however, the conformity-score distribution depends on the observed covariates,
making a single global threshold overly conservative. Instead, one seeks \emph{conditional}
coverage by learning a threshold that varies with the covariates.

Accordingly, in Model III the feature-score pairs satisfy
\[
(X_t,r_t)\sim P_t,
\qquad t=1,\ldots,T,
\]
where each $P_t$ is an unknown joint distribution on
$\mathcal X\times\mathbb R$. Equivalently,
\[
r_t\,|\,X_t
\sim
P_t(\cdot\,|\,X_t),
\]
and the conditional distributions are allowed to evolve over time.

At each round $t$, the learner first observes
\[
X_t\in\mathcal X=[0,1]^d,
\]
and then predicts using a threshold function
\[
q_t:\mathcal X\rightarrow\mathbb R,
\]
through the prediction set
\[
C_t(X_t)
=
\{y:s(X_t,y)\le q_t(X_t)\}.
\]
The objective is to satisfy the conditional coverage requirement
\[
\mathbb P(Y_t\in C_t(X_t)\mid X_t=x)
\ge
1-\alpha,
\qquad
\forall x\in\mathcal X.
\]

Let
\[
F_t(q\mid x)
=
\bbP_{P_t}(r_t\le q\mid X_t=x)
\]
denote the conditional conformity-score distribution. The optimal dynamic benchmark is therefore the oracle threshold function
\[
u_t^*(x)
=
F_t^{-1}(1-\alpha\mid x),
\]
which varies over both the covariate space and time and  dynamic regret is \eqref{defn:regret} with $u_t^*(x)
=
F_t^{-1}(1-\alpha\mid x)$. Unlike Model II, where the learner tracks a single scalar threshold, Model III requires tracking an evolving function.

The corresponding coverage violation function is
\[
c_t(q_t)
=
\mathbb E\!\left[
\left(
1-\alpha-
F_t(q_t(X_t)\mid X_t)
\right)_+
\right],
\]
and
\begin{equation}\label{defn:coverageprob2}
Q(T)
=
\sum_{t=1}^T
c_t(q_t).
\end{equation}

We assume that the oracle threshold function $u_t^*$ is $\beta$-Hölder continuous \cite{tsybakov2008introduction}, with path-length budget
\begin{definition}\label{defn:ST}
\[
{\cal S}_T
=
\sum_{t=1}^{T-1}
\sup_{x\in\mathcal X}
|u_{t+1}^*(x)-u_t^*(x)|.\]
\end{definition}
Using a Hölder-Partitioned ACI (HP-ACI) algorithm over grid cells of side length $h$, we prove
\[
Q(T),\,R(T)
=
O\!\left(
T^{\frac{5\beta+d}{6\beta+d}}
{\cal S}_T^{\frac{2\beta}{6\beta+d}}
\right).
\]

\begin{remark}[A unified one-sided/two-sided coverage functional]
\label{rem:unified-coverage}
The three coverage functionals $c_t(\cdot)$ used in
Models I-III are instances of a single template
\[
c_t(q_t) \;=\; \mathbb E\big[\rho(\Delta_t(q_t))\big],
\qquad
\Delta_t(q_t) :=
\begin{cases}
r_t - q_t, & \text{Model I (deterministic)},\\[2pt]
(1-\alpha) - F_t(q_t), & \text{Model II},\\[2pt]
(1-\alpha) - F_t(q_t(X_t)\mid X_t), & \text{Model III},
\end{cases}
\]
with $\rho=\rho_+(u):=u_+$ (one-sided shortfall) in Models I and III, and
$\rho=\rho_{\mathrm{abs}}(u):=|u|$ (two-sided) in Model II. Since
$\rho_+\le\rho_{\mathrm{abs}}$ pointwise, the Model I/III choice is the
\emph{weaker} individual statement; it suffices there only because the paired
dynamic regret $R(T)$ already penalizes the opposite side of the deviation
(over-coverage inflates set size). Model II reports coverage and efficiency as
formally separate quantities against a comparator $F_t^{-1}(1-\alpha)$ with no
efficiency-side penalty for overshoot, so the two-sided metric is needed there to
rule out trivial gaming (e.g.\ always playing $q_t=Q_{\max}$).

\end{remark}

{\bf Summary.}
Together these results establish a unified optimization theory for online conformal prediction. Beyond recovering guarantees across increasingly realistic statistical models, they demonstrate that coverage and prediction-set efficiency can be analyzed simultaneously against dynamic benchmarks, providing a principled framework for uncertainty quantification under distribution shift.

%

\paragraph{Notation and conventions.}
Throughout: $\alpha\in(0,1)$ is the target miscoverage level;
$A=[0,Q_{\max}]$ is the threshold action space;
$\X=[0,1]^d$ is the covariate space; $T$ is the time horizon.
We define
\[
z_t := \mathbf{1}(Y_t\in C_t) = \mathbf{1}(s(X_t,Y_t)\le q_t(X_t)),
\]
so $z_t=1$ means ``covered'' and $\E[z_t|q_t,X_t]=F_t(q_t(X_t)|X_t)$.
\section{Model I: Adversarial Scores}
\label{sec:adversarial}


The model is as defined in Section \ref{subsec:model1}, and the objective is to minimize the dynamic regret \eqref{defn:regretwc} and the cumulative coverage violation \eqref{defn:coveragewc} simultaneously for every adversarial sequence $r_1,\ldots,r_T$.

In this model, 
\[
z_t=\mathbf 1(r_t\le q_t)
\]
denote the coverage indicator. The learner updates the threshold using the classical Adaptive Conformal Inference (ACI) rule of \cite{gibbs2021adaptive},
\begin{equation}
\label{defn:aciupdate}
q_{t+1}
=
\Pi_A\!\left[
q_t+\eta\bigl((1-z_t)-\alpha\bigr)
\right],
\end{equation}
where $\eta>0$ is the step size.

\begin{remark}[Optimal Dynamic Benchmark]
\label{rem:comp}
Since $u_t^*=r_t$, the path length of the optimal dynamic benchmark is
\[
{\cal P}_T
=
\sum_{t=2}^T
|u_t^*-u_{t-1}^*|
=
\sum_{t=2}^T
|r_t-r_{t-1}|.
\]
\end{remark}
\begin{theorem}
\label{thm:adversarial}
Let $\Lambda:A\rightarrow\mathbb R$ be any monotone nondecreasing $L$-Lipschitz function (not necessarily convex). Then the ACI update \eqref{defn:aciupdate} satisfies
\[
Q(T)
=
O\!\left(\sqrt{T(1+{\cal P}_T)}\right),
\qquad
R(T)
=
O\!\left(L\sqrt{T(1+{\cal P}_T)}\right).
\]
\end{theorem}

\begin{remark}
\label{rem:one-bit-model1}
Although Model I is stated with the score $r_t$ realized every round, ACI update \eqref{defn:aciupdate} only depends on the round-$t$
outcome only through $z_t$. Consequently the $O(\sqrt{T(1+P_T)})$
guarantee below is a single-bit-feedback guarantee.
\end{remark}

\begin{proof}(Theorem \ref{thm:adversarial})

\begin{definition}[Pinball (quantile) loss]\label{defn:pinball}
\[
\psi_t(q):=(1-\alpha)(r_t-q)_++\alpha(q-r_t)_+.
\]
\end{definition}

The subdifferential of $\psi_t$ at $q$ is
\[
\partial\psi_t(q)=\begin{cases}
\{-(1-\alpha)\}&q<r_t,\\
[-(1-\alpha),\alpha]&q=r_t,\\
\{\alpha\}&q>r_t.
\end{cases}
\]
Define $\tilde{g}_t:=-((1-z_t)-\alpha)$. We next show that $\tilde{g}_t  \in\partial\psi_t(q_t)$.
\begin{itemize}
\item $q_t<r_t$: then $z_t=0$ (not covered), so
      $\tilde{g}_t=-(1-\alpha)\in\partial\psi_t(q_t)$. 
\item $q_t>r_t$: then $z_t=1$ (covered), so
      $\tilde{g}_t=\alpha\in\partial\psi_t(q_t)$. 
\item $q_t=r_t$: then $z_t=1$, so $\tilde{g}_t=\alpha$.
      Since $\alpha\in[-(1-\alpha),\alpha]=\partial\psi_t(r_t)$. 
\end{itemize}
Thus, the ACI update in \eqref{defn:aciupdate} can also be written as  
\begin{equation}\label{aci-ogd}
q_{t+1}=\Pi_A[q_t-\eta\tilde{g}_t],
\end{equation}

Supposing the objective was to minimize the dynamic regret 
\[
\sum_t\psi_t(q_t)-\sum_t\psi_t(u_t),\] 
for any comparator sequence $u_1, \dots, u_T$, following algorithm \eqref{aci-ogd} to generate $q_t$ is known \cite{zhang2018adaptive} to have 
\[
\sum_t\psi_t(q_t)-\sum_t\psi_t(u_t)
\le \frac{7Q_{\max}^2}{4\eta}+\frac{\eta}{2}T
    +\frac{Q_{\max}{\cal P}_T(u)}{\eta},
\]
where ${\cal P}_T(u) = \sum_{t=2}^T |u_t-u_{t-1}|$ and when \begin{itemize}
\item domain $\cal D$ of action space  is bounded,
\item $\psi_t$ is convex for all $t$, 
\item subgradients of $\psi_t$ are bounded for all $t$ in the domain,
\item $u_t \in \cal D$. 
\item ${\cal P}_T(u)$ is known ahead of time. When ${\cal P}_T(u)$ is not known, using multiple experts with different step-sizes and using Hedge gives the same result with only extra $\log T$ terms in regret \cite{zhang2018adaptive}.
\end{itemize} 

Choosing $\eta=O\left(\sqrt{\frac{1+{\cal P}_T(u)}{ T}}\right)$\footnote{This requires the knowledge of ${\cal P}_T(u)$. However, even without it, the same guarantee with extra $\log T$ factors is achievable by using multiple step sizes in parallel and using Hedge over it \cite{zhang2018adaptive}.}  gives 

\begin{equation}\label{eq:dynregret}
\sum_t\psi_t(q_t)-\sum_t\psi_t(u_t) \le 
O(\sqrt{T(1+{\cal P}_T(u))}).
\end{equation}

Since all the above listed requirements are satisfied in our problem, and the comparator of interest is $u_t^*=r_t$ i.e. ${\cal P}_T(u) ={\cal P}_T(u^*) = {\cal P}_T$, we have 
\begin{equation}\label{eq:finalregretboundmodel1}
\sum_t\psi_t(q_t)-\sum_t\psi_t(r_t) \le 
O(\sqrt{T(1+{\cal P}_T)}).
\end{equation}

Next, we  translate the pinball regret to bound $Q(T)$ and $R(T)$.

{\bf Coverage} With $\psi_t(r_t)=0$, from \eqref{eq:finalregretboundmodel1},
we get  $$\sum_t\psi_t(q_t)=O(\sqrt{T(1+P_T)}).$$
From the pinball loss definition \eqref{defn:pinball}, 
$\psi_t(q)\ge(1-\alpha)(r_t-q)_+$, so
$(r_t-q_t)_+\le\psi_t(q_t)/(1-\alpha)$.
Summing: $$Q(T)=\sum_{t=1}^T(r_t-q_t)_+\le\frac{1}{1-\alpha}\sum_t\psi_t(q_t)=
O(\sqrt{T(1+P_T)}).$$

{\bf Efficiency} Since $r_t=\arg\min_{q\ge r_t}\Lambda(q)$ by
monotonicity of $\Lambda$, for any $q_t$:
$$\Lambda(q_t)-\Lambda(r_t)\le L|q_t-r_t|.$$
From the pinball loss definition \eqref{defn:pinball} and using $(r_t-q)_+\le\frac{1}{1-\alpha}\psi_t(q)$ and
$(q-r_t)_+\le\frac{1}{\alpha}\psi_t(q)$, we get  $$|q-r_t|\le\frac{1}{\min(\alpha,1-\alpha)}\psi_t(q).$$
Thus, $$R(T)=
\sum_{t=1}^{T}
|\Lambda(q_t)
-
\Lambda(r_t)|\le\frac{L}{\min(\alpha,1-\alpha)}\sum_t\psi_t(q_t)
=O(L\sqrt{T(1+P_T)}).$$
\end{proof}

Following remarks are in order.

\begin{remark}
The standard ACI guarantee \cite{gibbs2021adaptive} controls the signed average
$\frac{1}{T}\sum_t(z_t-\alpha)\to 0$, while our Theorem~\ref{thm:adversarial} bounds absolute error 
 for both
coverage and efficiency.
\end{remark}

\begin{remark}[Decoupling the surrogate from the objective]
Note that the regret guarantee $R(T)$ is with respect to $\Lambda$, which is not
necessarily convex. Thus, we cannot import known results on dynamic regret
for online convex optimization (OCO) directly.
To derive the result, we exploit the observation that the ACI update is
exactly a projected online gradient descent step on the pinball loss $\psi_t$.
All of the online learning is {\it deemed} to be performed on the pinball loss $\psi_t$, a convex surrogate that does not depend on $\Lambda$ at all. 

The dynamic regret bound
 for $\psi_t$ with respect to the
the pointwise-optimal comparator $u_t^*=r_t$ is then transferred to $\Lambda$ using the specific form of $\psi_t$ and its connection to $\Lambda$ via $|q-r_t|$.
This decoupling, optimize a fixed convex surrogate, then transfer the
bound to an arbitrary monotone Lipschitz objective after the fact, is
what allows the guarantee to hold for \emph{any} such $\Lambda$
simultaneously, including non-convex ones, and is what existing
regret-to-coverage reductions and static-benchmark efficiency guarantees
do not provide. The theorem does not require a new algorithm or a novel regret-analysis technique, its real contribution is this reduction and making the necessary connection between dynamic regret wrt to $\Lambda$ and $\psi_t$.
\end{remark}

\begin{remark}[Relation to regret-to-coverage reductions and parameter-free OCP]
\label{rem:regret-to-coverage}
The identity underlying Theorem~\ref{thm:adversarial}, that ACI is exactly projected
online gradient descent on the pinball loss, places our adversarial result
within the broader family of \emph{regret-to-coverage reductions}, which convert a
black-box no-regret online-learning guarantee into a calibration guarantee for
online conformal prediction. This family includes the strongly-adaptive reduction
of Bhatnagar et al.~\cite{bhatnagar2023improved}, and the integral-control formulation of
Angelopoulos, Cand\`es, and Tibshirani~\cite{angelopoulos2023pid}, which recovers
ACI as a proportional controller and shows that more general PID-style controllers
inherit no-regret guarantees against dynamic comparators. Relative to this family, Theorem~\ref{thm:adversarial} makes a stronger joint
claim. The \emph{same} regret certificate~\eqref{eq:finalregretboundmodel1} is
used to bound both quantities: it certifies coverage directly, and, via the
efficiency paragraph of the proof, it also certifies prediction-set efficiency
against the dynamic oracle $r_t$. Regret-to-coverage reductions do not provide
this joint statement as they stand. They are typically analyzed purely in
pinball-loss units, with no step that translates the regret bound into
set-size units via the monotonicity of $\Lambda$.

A more substantively different line of work targets conditional or
multi-group coverage in the \emph{adversarial} setting via swap-regret
minimization \cite{blum2007external, gibbs2023conditional}, rather than the
single-scalar threshold tracked in Model~I. These methods certify coverage
simultaneously across an arbitrary, possibly adversarially and adaptively
revealed, collection of (possibly overlapping) groups, generalizing beyond the
fixed covariate-partition structure of Model~III. Combining swap-regret-style
group-conditional tracking with the monotone-Lipschitz efficiency argument of
Theorem~\ref{thm:adversarial}, applied within each adversarially-selected group, would
yield an ``adversarial Model~III'' not covered by the present framework: our
Model~III instead relies on the stochastic contraction argument of
Section~\ref{sec:conditional}, which uses the density floor $f_{\min}^s$ to
guarantee a restoring force (Assumption C4 in Section~\ref{sec:conditional}) and has no direct
analogue when scores within a group are adversarial rather than drawn from a fixed
conditional distribution. We flag this as a direction for future work rather than
a result of the present paper.
\end{remark}

\section{Model II: Stochastic Scalar Scores}
\label{sec:scalar}
The setup is as defined in Section \ref{subsec:model2}, and the objective is to minimize the dynamic regret \eqref{defn:regret} with $u_t^*
=
F_t^{-1}(1-\alpha)$, and the cumulative coverage violation \eqref{defn:coverageprob1}, simultaneously.

%


We work under the following standard assumptions. 
\begin{assumption}\label{ass:scalar}
\begin{enumerate}[label=(A\arabic*)]
\item \textbf{Independence.} The pairs $(X_t,Y_t)$ are independent
      across $t$, each drawn from $P_t$; equivalently, the induced
      conformity scores $r_t=s(X_t,Y_t)$ are independent across $t$,
      with $r_t\sim F_t$.
\item \textbf{Density lower bound.} There exists $f_{\min}>0$ such that
      $F_t'(q)\ge f_{\min}$ for all $q\in A=[0,1]$, $t\ge 1$.
      This ensures $F_t$ is strictly increasing, so $u_t^*$ is uniquely
      defined and the inverse map is $1/f_{\min}$-Lipschitz.
\item \textbf{Uniform concentration.} There exists $C_{VC}>0$ such that
      for every $W\ge 1$ and $t>W$,
      \[
        \E\!\left[\sup_{q\in A}
        \left|\widehat{F}_{t-1,W}(q)-\bar{F}_{t-1,W}(q)\right|\right]
        \le\frac{C_{VC}}{\sqrt{W}},
      \]
      where $\widehat{F}_{t-1,W}(q)=\frac{1}{W}\sum_{\tau=t-W}^{t-1}
      \mathbf{1}(r_\tau\le q)$ is the empirical CDF and
      $\bar{F}_{t-1,W}(q)=\frac{1}{W}\sum_{\tau=t-W}^{t-1}F_\tau(q)$
      is its expectation.
\end{enumerate}
\end{assumption}

\begin{remark}
\label{rem:dkw}
Assumption A3 follows from the \textbf{Dvoretzky--Kiefer--Wolfowitz
(DKW) inequality} \cite{massart1990} and does not need to be assumed
separately. Specifically, for each $\tau$, $\mathbf{1}(r_\tau\le q)$
is a Bernoulli random variable with mean $F_\tau(q)$; since scores are
independent across $\tau$ (Assumption A1), the window
$\{r_{t-W},\dots,r_{t-1}\}$ consists of $W$ \emph{independent} (but
not identically distributed) observations. The DKW inequality for n 
i.i.d.\ (independent non-identically distributed) samples gives
\[
  \E\!\left[\sup_q|\widehat{F}_{t-1,W}(q)-\bar{F}_{t-1,W}(q)|\right]
  \le\sqrt{\frac{\pi}{2W}},
\]
so Assumption A3 holds with $C_{VC}=\sqrt{\pi/2}$. We retain A3 as a
formal assumption to make the constant explicit and to allow tighter
bounds from other concentration tools if desired.
\end{remark}

\subsection{Algorithm: Sliding-Window Empirical Quantile}
Without loss let $Q_{\max} = 1$ and $A = [0,1]$.
\paragraph{Initialization.} For $t=1,\dots,W$, play any $q_t\in A$
and observe $r_t$. 

\paragraph{Update.} For every $t>W$:
\begin{enumerate}[label=(\arabic*)]
\item Build the sliding-window empirical CDF:
      $\widehat{F}_{t-1,W}(q)=\frac{1}{W}\sum_{\tau=t-W}^{t-1}
      \mathbf{1}(r_\tau\le q)$.
\item Set $q_t=\inf\{q\in A:\widehat{F}_{t-1,W}(q)\ge 1-\alpha\}$.
\item Play $q_t$, observe $r_t$.
\end{enumerate}

\subsection{Upper Bound}

\begin{theorem}
\label{thm:scalar-upper}
Under Assumption~\ref{ass:scalar}, with window size

\[W^\dagger = \min\!\left(\left\lceil \left(\frac{C_{VC}\,T}{2(1+V_T)}\right)^{2/3}\right\rceil,\; T\right)\]
the sliding-window algorithm, for $R(T)$  \eqref{defn:regret} with $u_t^*
=
F_t^{-1}(1-\alpha)$, and  $Q(T)$ \eqref{defn:coverageprob1}, we have
\begin{align*}
R(T) = O\!\left(\left(\frac{T}{1+V_T}\right)^{2/3} + \frac{T^{2/3}(1+V_T)^{1/3}}{f_{\min}}\right),\qquad
Q(T) = O\!\left(\left(\frac{T}{1+V_T}\right)^{2/3} + T^{2/3}(1+V_T)^{1/3}\right). 
\end{align*}
\end{theorem}
The proof is provided in Section \ref{sec:ProofModel2}.
\begin{remark}
The sliding-window algorithm requires the knowledge of $V_T$ which is not available causally. In Section \ref{sec:hedge_windows}, we show how to remove that requirement but with an inferior guarantee. 
\end{remark}

\begin{remark}[Density lower bound: necessity and localization]
\label{rem:density-necessity}
Assumption~A2 is used exclusively to convert the coverage-level guarantee
of Lemma~\ref{lem:qcl} into the threshold-space guarantee needed for
$R(T)$; it plays no role in bounding $Q(T)$, which holds under
Assumptions~A1 and~A3 alone. The dependence of the bound on $1/f_{\min}$ is unavoidable rather than a
proof artifact: if $F_t$ has vanishing density on some interval, two
thresholds $q,q'$ in that interval can be observationally
indistinguishable from coverage feedback ($F_t(q)\approx F_t(q')$) while
$\Lambda(q)-\Lambda(q')$ is arbitrarily large, so no algorithm can control $R(T)$
without some assumption relating $F$-space to $q$-space near the oracle
threshold.

Assumption~A2 cannot simply be weakened to a \emph{local} density bound
at $u_t^*$ alone. The Mean Value Theorem point $\tilde q$ produced in
Step~3 of the proof of Theorem~\ref{thm:scalar-upper} lies between
$q_t$ and $u_t^*$, so confining $\tilde q$ to a neighborhood
$[u_t^*-\delta,u_t^*+\delta]$ is equivalent to confining $q_t$ itself,
i.e.\ to already knowing $|q_t-u_t^*|\le\delta$. But bounding
$|q_t-u_t^*|$ in terms of the $F$-space error
$\Delta_{t,W}+\mathcal E_t+1/W$ is exactly what Step~3 uses the
density lower bound to establish in the first place: smallness of
$|F_t(q_t)-(1-\alpha)|$ forces $q_t$ near $u_t^*$ only because $f_t$ is
bounded below, and by the same counterexample as above this can fail
completely if the density vanishes somewhere between $q_t$ and
$u_t^*$. A bound on $f_t$ only \emph{inside} $[u_t^*-\delta,u_t^*+\delta]$
gives no control on how far $q_t$ can drift from $u_t^*$ before entering
that neighborhood, so confinement cannot be deduced from the local bound
alone; asserting it directly from $o(\delta)$-smallness of the
$F$-space error, without further input, is circular.

Repairing this requires an additional ingredient beyond a purely local
bound at $u_t^*$. One option is a bound on $f_t$ that need not be
uniform over all of $A$ but must at least cover a fixed corridor
containing every plausible value of $q_t$ prior to confinement, e.g.\ a
floor $f_{\min}^{\mathrm{loc}}$ on $[u_t^*-\delta_0,u_t^*+\delta_0]$ for
some \emph{a priori} $\delta_0$ that is not itself derived from the
estimation-error bound, together with a separate argument (independent
of the density floor) showing $q_t$ cannot exit $[u_t^*-\delta_0,u_t^*+\delta_0]$
after the stated burn-in period. The sliding-window estimator recomputed
from scratch each round, $q_t=\inf\{q:\widehat F_{t-1,W}(q)\ge 1-\alpha\}$,
has no built-in inertia linking $q_t$ to $q_{t-1}$, so such a corridor
bound would need to come from elsewhere (e.g.\ a global, possibly weak,
density floor outside $[u_t^*-\delta_0,u_t^*+\delta_0]$, ruling out
$\widehat F_{t-1,W}$ crossing $1-\alpha$ far from $u_t^*$). Under such a
corridor assumption, once $\Delta_{t,W}+\mathcal E_t+1/W=o(\delta_0)$ for
all $t$ beyond a burn-in period determined by $W$ and $V_T$, the argument
of Step~3 goes through with $f_{\min}$ replaced by
$f_{\min}^{\mathrm{loc}}$, contributing only an additive burn-in term to
$R(T)$ and leaving the leading-order rate
$O\big(T^{2/3}(1+V_T)^{1/3}/f_{\min}^{\mathrm{loc}}\big)$ unaffected. Without
some such corridor-type assumption, however, the localization claim as
originally stated does not follow from the $F$-space bound alone.
\end{remark}

\subsection{Minimax Lower Bound}

\paragraph{The fundamental identification problem.}
Before stating the lower bound, it helps to identify the core
statistical challenge. An algorithm choosing $q_t$ at each round faces
an \emph{estimation problem}: it must distinguish the current oracle
threshold $u_t^*=F_t^{-1}(1-\alpha)$ from a nearby alternative, using
only a finite window of past scores. The key difficulty is that
\emph{any two distributions that are close in KL divergence require many
samples to tell apart}, yet the oracle threshold can change between
them. The lower bound formalizes this by constructing a two-distribution
``hypothesis testing'' problem where: the oracle thresholds differ by
$\Delta$, but the distributions are so similar that $B\sim 1/\Delta^2$
samples are needed to distinguish them. During these $B$ rounds the
algorithm is essentially guessing, paying at least $\Delta/2$ in
tracking error per round. Summing over $K=T/B$ blocks gives the
$\Omega(T\Delta)$ lower bound, and optimizing $\Delta$ against $V_T$
yields $T^{2/3}V_T^{1/3}$.

\begin{theorem}[Minimax lower bound]
\label{thm:scalar-lower}
For every online algorithm $\mathcal{A}$ and sufficiently small
universal constant $c_2>0$, for any $V_T\le c_2 T$:
\[
\inf_{\mathcal{A}}\sup_{\{F_t\}\in\mathcal{D}(V_T)}R(T)\ge c_1T^{2/3}V_T^{1/3},
\quad
\inf_{\mathcal{A}}\sup_{\{F_t\}\in\mathcal{D}(V_T)}Q(T)\ge c_1f_{\min}T^{2/3}V_T^{1/3},
\]
where $c_1>0$ is a universal constant and $\mathcal{D}(V_T)$ is the class of distributions with variation budget at most $V_T$. 
\end{theorem}

\begin{remark}[The $Q(T)$ lower bound is not stated tightly]
The $R(T)$ lower bound in Theorem~\ref{thm:scalar-lower} is established
directly, in threshold space, and involves no density constant. The
$Q(T)$ bound is instead obtained from it via the generic inequality
$|F_t(q)-F_t(q^*)|\ge f_{\min}|q-q^*|$, valid for any $F_t$ satisfying
Assumption~A2, applied to the same hard instance used for $R(T)$. That
instance itself has density $\Theta(1)$ rather than $f_{\min}$, so this
conversion step is loose whenever $f_{\min}\ll 1$: it uses the worst-case
floor permitted by the model class rather than the actual steepness of
the constructed distributions. Since the upper bound of
Theorem~\ref{thm:scalar-upper} shows $Q(T)=O(T^{2/3}(1+V_T)^{1/3})$ with
no dependence on $f_{\min}$, we conjecture the true minimax rate for
$Q(T)$ is $\Omega(T^{2/3}V_T^{1/3})$ without the $f_{\min}$ factor. \end{remark}

\section{Model III: Stochastic Conditional Scores}
\label{sec:conditional}
The setup is as defined in Section \ref{subsec:model3} and the objective is to simultaneously minimize \eqref{defn:regret} 
with $
u_t^*(x)
=
F_t^{-1}(1-\alpha\mid x)$ and coverage \eqref{defn:coverageprob2}.
We work under the following assumptions.
\subsection{Assumptions}

\begin{assumption}[Conditional setting]
\label{ass:conditional}
\begin{enumerate}[label=(C\arabic*)]
\item \textbf{H\"{o}lder oracle.}
      $u_t^*\in\mathcal{Q}_\beta(L)
      =\{u:\X\to A:|u(x)-u(x')|\le L\|x-x'\|^\beta\}$
      for fixed $\beta\in(0,1]$, $L>0$, every $t$. This says the
      oracle threshold varies smoothly in the covariate space.

\item \textbf{Spatial score continuity.}
      $|F_t(q\mid x)-F_t(q\mid x')|
      \le L\|x-x'\|^\beta$
      uniformly in $q$ and $t$.

\item \textbf{Oracle path-length budget.}
      \[
      {\cal S}_T
      :=
      \sum_{t=1}^{T-1}
      \sup_{x\in\X}
      |u_{t+1}^*(x)-u_t^*(x)|
      <\infty.
      \]

\item \textbf{Independence and density bounds.}
      $(X_t,Y_t)$ are independent across $t$;
      $f_t(x)\in[f_{\min},f_{\max}]$;
      the conditional conformity-score density satisfies
      $f_t(q\mid x)\in[f_{\min}^s,f_{\max}^s]$
      for all $q\in A$, $t$, and $x$.
      The global density floor $f_{\min}^s>0$
      is needed for iterate confinement
      (see Remark~\ref{rem:confinement}).

\item \textbf{Step-size bound.}
      \[
      \eta\le\frac{f_{\min}^s}{2(f_{\max}^s)^2}.
      \]
      This ensures that the contraction coefficient lies in
      $[0,1)$ at every step, preventing the iterate from
      oscillating past the oracle.

      \end{enumerate}
\end{assumption}

\subsection{Algorithm: H\"older-Partitioned ACI (HP-ACI)}

\paragraph{Basic idea.}
Because the oracle threshold $u_t^*(x)$ varies with $x$,
a single scalar ACI instance cannot track it everywhere
simultaneously.
We partition the covariate space
$\X=[0,1]^d$
into cells of side length $h$
and run one independent ACI instance in each cell,
updating only the cell visited by the current covariate.
Cells that are visited infrequently receive fewer updates but also
contribute less to the overall coverage error.
The bandwidth $h$ balances statistical estimation
(small cells imply fewer observations per cell)
against spatial approximation
(large cells incur greater discretization error).

\paragraph{Algorithm}
Cover $\X$ by
$N=\lceil h^{-d}\rceil$
axis-aligned cells
$\{B_j\}_{j=1}^N$
of side length $h$,
with representative point
$x_j\in B_j$.
Maintain a threshold
$q_t^j\in A$
for each cell, and define the piecewise-constant threshold function
\[
q_t(x)=q_t^j,
\qquad x\in B_j.
\]
Recall that
\[
z_t
=
\mathbf1\!\left(
s(X_t,Y_t)\le q_t(X_t)
\right).
\]

\paragraph{Update.}
At round $t$, let
$j=j(t)$
denote the cell containing $X_t$. That cell is also defined to be the {\it visited} cell.
The visited cell is updated according to
\begin{equation}\label{defn:acicell}
q_{t+1}^{j(t)}
=
\Pi_A\!\Bigl[
q_t^{j(t)}
+
\eta\bigl((1-z_t)-\alpha\bigr)
\Bigr],
\end{equation}
while all other cells remain unchanged,
\[
q_{t+1}^{j'}
=
q_t^{j'},
\qquad
\forall\,j'\neq j(t).
\]

The update direction
$(1-z_t)-\alpha$
has conditional mean
\[
F_t(q_t^{j(t)}\mid X_t)
-
(1-\alpha),
\]
given
$q_t^{j(t)}$
and
$X_t$,
which vanishes exactly at the oracle threshold
$u_t^*(X_t)$.
Thus, each cell performs a Robbins-Monro stochastic approximation
toward its local oracle threshold.

\begin{remark}[Why a global density floor is needed]
\label{rem:confinement}
The global density bound in Assumption~C4
($f_{\min}^s>0$ over the entire interval $A$, not merely in a
neighborhood of the oracle threshold $u_t^*$)
is essential for the proof.
Without it, if the iterate $q_t^j$ drifts away from $u_t^*$,
the ACI update \eqref{defn:acicell} may no longer possess a restoring force,
since the Mean Value Theorem argument relies on
$f_t(\xi\mid x)\ge f_{\min}^s$
for every point $\xi$ lying between the iterate and the oracle.
If $\xi$ enters a region where the density vanishes,
the contraction argument breaks down and the recursion may diverge.
The global lower bound $f_{\min}^s>0$
(which may be much smaller than the local density near the oracle)
guarantees a weak but uniform restoring force throughout the action
space. Its value affects only the burn-in period and not the final
asymptotic rate.
\end{remark}

\subsection{Main Result}

\begin{theorem}[Conditional tracking bound]
\label{thm:conditional}
Under Assumption~\ref{ass:conditional}, choose
\[
h^*=T^{-1/(6\beta+d)}{\cal S}_T^{2/(6\beta+d)}
\]
and the global step size
\[
\eta^*
=
\Theta\!\left(
\left(
\frac{{\cal S}_T^2}
{Th^d}
\right)^{1/3}
\right)
\]
that satisfies Assumption~C5, where ${\cal S}_T$ is as defined in Definition \eqref{defn:ST}. 
Then HP-ACI satisfies
\[
Q(T)
=
O\!\left(
T^{\frac{5\beta+d}{6\beta+d}}
{\cal S}_T^{\frac{2\beta}{6\beta+d}}
\right),
\quad
\text{and} 
\quad
R(T)
=
O\!\left(
\frac{L}{f_{\min}^s}
T^{\frac{5\beta+d}{6\beta+d}}
{\cal S}_T^{\frac{2\beta}{6\beta+d}}
\right).
\]
\end{theorem}

\begin{remark}[Unknown
$f_{\min}^s, f_{\max}^s$] 
\label{rem:diminishing-stepsize} Algorithm HP-ACI needs to satisfy Assumption~(C5) even though $f_{\min}^s$ and $f_{\max}^s$ are not always known.
Assumption~(C5) is a \emph{stability} requirement, not merely a rate-tuning
choice: it guarantees $\rho_k = 1-\eta\, p_{t_k}(\xi_k\mid x_j) \in [0,1)$ so
that the per-cell recursion contracts. Because the iterate is confined to set
$A$ by projection regardless of $\eta$, an oversized step size degrades the
\emph{rate} of Theorem~\ref{thm:conditional} rather than causing divergence,
so the simplest resolution to remove the knowledge of $f_{\min}^s$ and $f_{\max}^s$ is to give up a fixed, tuned $\eta$ altogether
and let it shrink over time.

Replace the fixed global step size $\eta^\star$ in each cell by a diminishing
schedule $\eta_k = c\,k^{-\theta}$, $\theta\in(1/2,1)$, applied to the
$k$-th visit of that cell, for any constant $c>0$. Since
$\eta_k\to 0$ as $k\to\infty$, there exists a (cell- and
distribution-dependent, but immaterial) visit count $k_0$ after which
$\eta_k \le f_{\min}^s/(2(f_{\max}^s)^2)$ automatically, for \emph{any}
finite $f_{\min}^s,f_{\max}^s>0$, without the learner ever needing to
know these constants. Rounds $k\le k_0$ contribute only an
additive, distribution-dependent burn-in cost to $V_c(T)$ and $R(T)$.

This sacrifices the optimal balance between estimation and drift that the
tuned $\eta^\star=\Theta\big((S_T^2/(Th^d))^{1/3}\big)$ achieves in
Theorem~\ref{thm:conditional}: the standard Robbins-Monro analysis of the
Lyapunov recursion~\eqref{eq:cond-lyap} with $\eta_k=ck^{-\theta}$
yields a per-cell cumulative error of order $m^{1-\theta/2}$ rather than the
optimal $m^{2/3}$, so the resulting bound on $Q(T),R(T)$ is sublinear but
no longer rate-optimal in $S_T$. The qualitative guarantee, sublinear
coverage violation and dynamic regret, is preserved with no knowledge of
$f_{\min}^s,f_{\max}^s$ required. Recovering the exact rate of
Theorem~\ref{thm:conditional} up to logarithmic terms without this knowledge would require the
Hedge-based construction of Remark~\ref{rem:adaptive-stepsize}, at the
price of additional algorithmic complexity.
\end{remark}

\begin{remark}[Adaptive Step-Size Selection: Removing Knowledge of $\mathcal{S}_T$]\label{rem:adaptive-stepsize} The optimal step size $\eta^\star$ derived in Step 3 relies on the oracle path length $\mathcal{S}_T$, which is unknown to the learner. Similar to the adaptive window selection used in Model II, this dependence can be eliminated via expert aggregation without degrading the primary asymptotic rate.Within each cell $j$, we maintain a localized Hedge (multiplicative weights) instance over a geometric grid of candidate step sizes $\mathcal{H} = \{2^{-i} : i=0, 1, \dots, \lceil \log_2 T \rceil\}$. At each visit to the cell, the active step size is drawn from the local Hedge distribution, and the weights are updated using the realized coverage loss $(1-z_k)-\alpha$. Running Hedge over $\vert{}\mathcal{H}\vert{} = O(\log T)$ experts incurs a meta-regret of $O(\sqrt{m_j \log \log T})$ per cell relative to the optimal fixed step size. Summing this adaptive cost across all $N \asymp h^{-d}$ cells on the high-probability event ($m_j \le p_{\max}h^dT$) yields an aggregate meta-regret of:$$O\!\left( h^{-d} \sqrt{h^dT \log \log T} \right) = O\!\left( T^{1/2} h^{-d/2} \sqrt{\log \log T} \right).$$Substituting the optimal bandwidth $h^* \asymp T^{-1/(6\beta+d)}$, the meta-regret scales as $\widetilde{O}(T^{\frac{1}{2} + \frac{d}{2(6\beta+d)}})$. Comparing this to the primary rate $O(T^{\frac{5\beta+d}{6\beta+d}})$, the adaptive cost is strictly lower order if and only if:$$\frac{1}{2} + \frac{d}{2(6\beta+d)} < \frac{5\beta+d}{6\beta+d} \iff 6\beta + 2d < 10\beta + 2d \iff \beta > 0.$$Because $\beta > 0$ is guaranteed by the Hölder continuity of the target function (Assumption~C1), the local step sizes can be made fully adaptive while preserving the $O(T^{\frac{5\beta+d}{6\beta+d}})$ tracking bound.\end{remark}

%
%
%
%
\section{Numerical Results}
In this section, we provide comprehensive numerical results to compare the performance of the different algorithm for the three models. We start with the simplest setting of real data used by \cite{gibbs2021adaptive}.

\subsection{Real-Data Replication (Gibbs--Cand\`es \cite{gibbs2021adaptive} Volatility Benchmark)}

\subsubsection{Simulation setup}
\label{sec:dax-setup}

The experiment is built to mirror, as closely as the available data permits, the
volatility-forecasting benchmark of Gibbs and Cand\`es \cite{gibbs2021adaptive}
(their Figure~1 and the \texttt{garchConformalForcasting} routine in the accompanying
code release, \url{https://github.com/isgibbs/AdaptiveConformal}), while adding the
absolute-violation and efficiency instrumentation introduced in the unified
online-conformal-prediction framework under study.

\paragraph{Data.} Gibbs and Cand\`es calibrate on a real daily financial return series
and do not distribute the series itself with their code (only the R routines are
public). As a real substitute reachable in this environment, we use the daily closing
prices of the DAX index from the \texttt{EuStockMarkets} dataset (base R
\texttt{datasets} package, 1991--1998, $1860$ trading days), obtained through the
public \texttt{Rdatasets} mirror. Log returns are computed as
\[
  r_t^{\mathrm{ret}} = 100\big(\log P_t - \log P_{t-1}\big), \qquad t = 1,\dots,T-1 .
\]

\paragraph{Base volatility model.} Exactly as in \cite{gibbs2021adaptive}, a
GARCH$(1,1)$ model is fit on a rolling lookback window and used to produce a
one-step-ahead volatility forecast $\hat\sigma_t$. We use lookback $=250$ days
(compared to $1250$ in the original, shortened because \texttt{EuStockMarkets} only
has $1860$ points total) and, for computational tractability, re-estimate the
GARCH parameters every $10$ days rather than every single day, reusing the fixed
parameters (updated with the newest window) for one-step forecasts in between; this
is a compute-saving approximation to the original daily-refit loop and is the only
material deviation from \cite{gibbs2021adaptive}'s procedure.

\paragraph{Conformity score.} Identical to \cite{gibbs2021adaptive}:
\begin{equation}
  r_t \;=\; \frac{\big|(r_t^{\mathrm{ret}})^2 - \hat\sigma_t^2\big|}{\hat\sigma_t^2}.
  \label{eq:score}
\end{equation}
This is exactly the ``NC''/``OC'' error-sequence construction of their
\texttt{garchConformalForcasting} function, with $\alpha = 0.10$.

\paragraph{Three online thresholding rules, run on the identical score sequence
\eqref{eq:score}.}
\begin{itemize}
  \item \textbf{Non-adaptive.} A single threshold $q_0$, the empirical
        $(1-\alpha)$-quantile of the first $100$ scores, held fixed for the
        remainder of the horizon. This is the ``no-adapt'' comparator already used
        in \cite{gibbs2021adaptive}.
  \item \textbf{Standard ACI (Model~I).} The Gibbs--Cand\`es update itself,
        \[
          q_{t+1} = \Pi_{\mathcal A}\Big[q_t + \gamma\big((1-z_t)-\alpha\big)\Big],
          \qquad z_t = \mathbb 1(r_t \le q_t),
        \]
        with step size $\gamma = 0.05$. This is simultaneously \cite{gibbs2021adaptive}'s
        original algorithm and Model~I of the unified framework, since the two are
        shown to be numerically identical.
  \item \textbf{Sliding-window quantile (Model~II).} $q_t = $ the empirical
        $(1-\alpha)$-quantile of the previous $W=150$ realized scores, recomputed
        every round.
\end{itemize}

\paragraph{Metrics.} For each method we track:
\begin{align}
  \textbf{Signed error \cite{gibbs2021adaptive} metric):}\quad
    & \bar e_t = \frac{1}{t}\sum_{s\le t}\big[(1-z_s)-\alpha\big], \\
  \textbf{Absolute violation:}\quad
    & Q(T) = \sum_{t\le T} (r_t - q_t)_+, \\
  \textbf{Efficiency functional:}\quad
    & \Lambda(q) = \sqrt{1+q}
    , \\
  \textbf{Dynamic oracle:}\quad
    & u_t^{*} = \text{quantile}_{1-\alpha}\big(\{r_s : |s-t|\le 75\}\big), \\
  \textbf{Efficiency regret:}\quad
    & R(T) = \sum_{t\le T}\big|\Lambda(q_t) - \Lambda(u_t^{*})\big|.
\end{align}
$u_t^*$ is (not causal, a benchmark only) deliberately allowed to see scores slightly into the future: it stands in
for the dynamically optimal comparator $u_t^* = F_t^{-1}(1-\alpha)$, which no causal
algorithm can access, exactly the role played by the oracle sequence in the paper's
regret definition.  Choice of $\Lambda(q) = \sqrt{1+q}$ is monotone, Lipschitz; proportional to interval width  $\hat\sigma_t\sqrt{1+q}$,

\subsubsection{Figure 1: coverage and efficiency, raw}

Three panels, one curve per method (gray = non-adaptive, blue = ACI/Model~I,
red = sliding-window/Model~II), all driven by the identical score sequence
\eqref{eq:score}: (top) the running signed miscoverage error $\bar e_t$ used by
\cite{gibbs2021adaptive}; (middle) the cumulative absolute violation $Q(T)$; (bottom)
the raw prediction-set-size trace $\Lambda(q_t)$.

\begin{figure}[h!]
\centering
\begin{tikzpicture}
\begin{axis}[
    width=13cm, height=4cm,
    ylabel={signed error},
    xlabel={},
    xticklabels={},
    grid=both, grid style={gray!20},
    legend style={font=\tiny, at={(1.01,0.5)}, anchor=west},
    legend cell align=left,
]
\addplot[cgray, thick] table[x=t,y=signed,col sep=comma] {Simulations/fig1_naive.csv};
\addplot[cblue, thick] table[x=t,y=signed,col sep=comma] {Simulations/fig1_aci.csv};
\addplot[cred,  thick] table[x=t,y=signed,col sep=comma] {Simulations/fig1_sw.csv};
\addplot[black, dashed] coordinates {(0,0) (1600,0)};
\legend{Non-adaptive, ACI (Model I), Sliding-window (Model II)}
\end{axis}
\end{tikzpicture}

\vspace{2mm}
\begin{tikzpicture}
\begin{axis}[
    width=13cm, height=4cm,
    ylabel={$Q(T)$},
    xticklabels={},
    grid=both, grid style={gray!20},
]
\addplot[cgray, thick] table[x=t,y=QT,col sep=comma] {Simulations/fig1_naive.csv};
\addplot[cblue, thick] table[x=t,y=QT,col sep=comma] {Simulations/fig1_aci.csv};
\addplot[cred,  thick] table[x=t,y=QT,col sep=comma] {Simulations/fig1_sw.csv};
\end{axis}
\end{tikzpicture}

\vspace{2mm}
\begin{tikzpicture}
\begin{axis}[
    width=13cm, height=4cm,
    ylabel={$\Lambda(q_t)$},
    xlabel={trading day},
    grid=both, grid style={gray!20},
]
\addplot[cgray, thin] table[x=t,y=eff,col sep=comma] {Simulations/fig1_naive.csv};
\addplot[cblue, thin] table[x=t,y=eff,col sep=comma] {Simulations/fig1_aci.csv};
\addplot[cred,  thin] table[x=t,y=eff,col sep=comma] {Simulations/fig1_sw.csv};
\end{axis}
\end{tikzpicture}
\caption{Raw coverage and efficiency traces on real DAX/GARCH scores.}
\label{fig:1}
\end{figure}
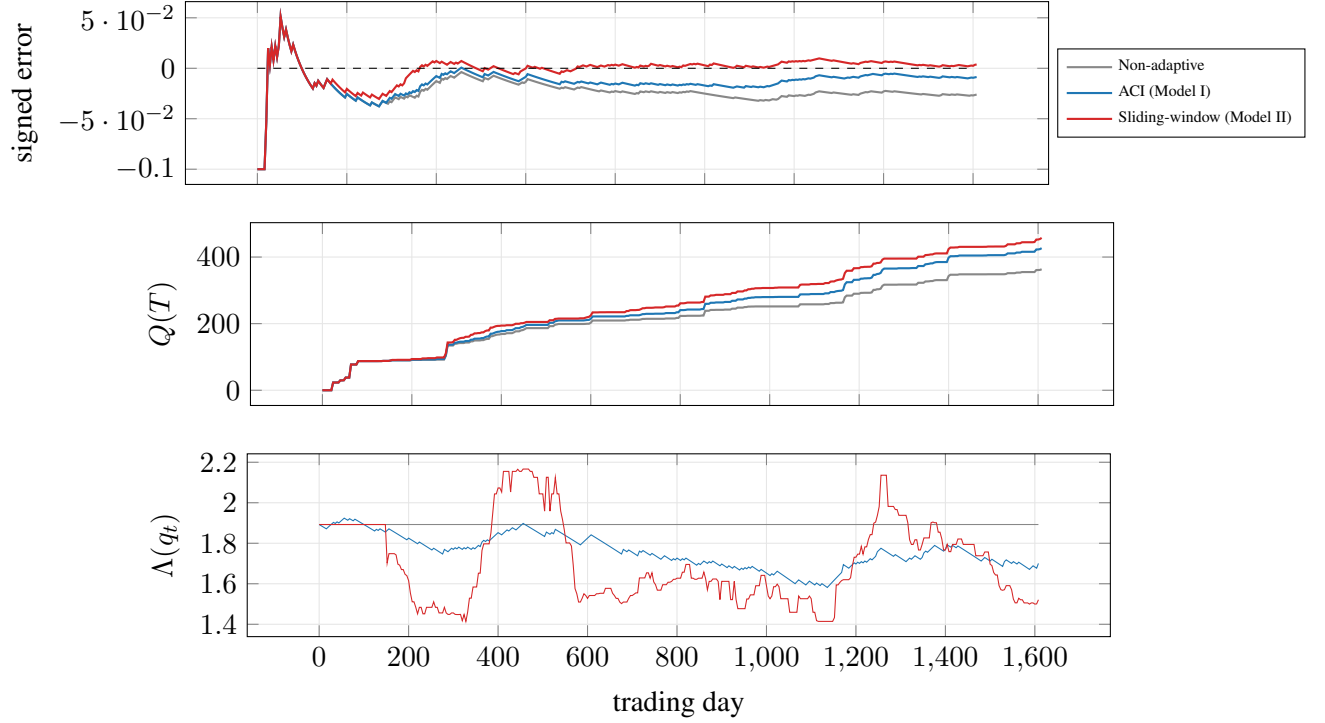

\subsubsection{Figure 2: efficiency against the dynamic oracle}

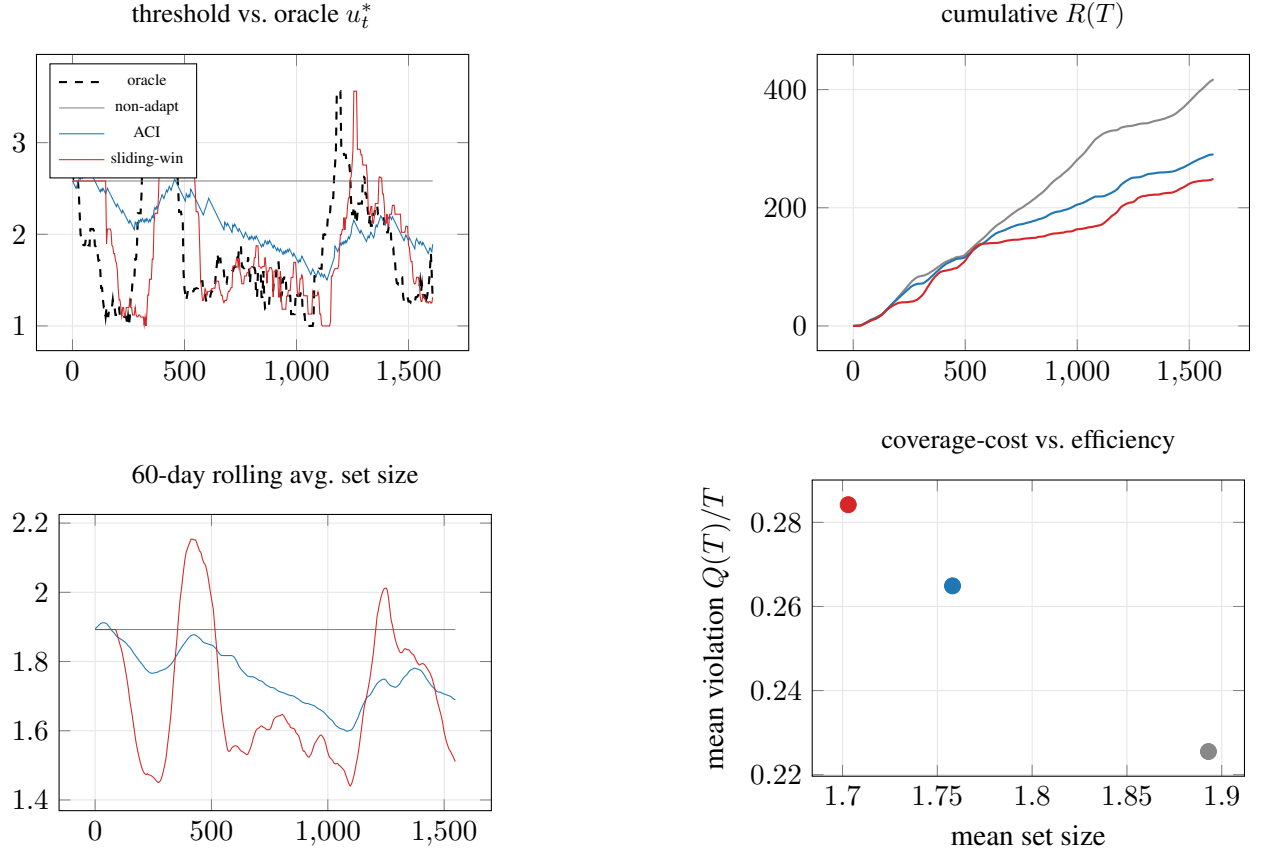
\begin{figure}[h!]
\centering
\begin{tikzpicture}
\begin{axis}[
    width=7.3cm, height=5.5cm,
    title={\small threshold vs.\ oracle $u_t^*$},
    grid=both, grid style={gray!20},
    legend style={font=\tiny}, legend pos=north west,
]
\addplot[black, dashed, thick] table[x=t,y=oracle_q,col sep=comma] {Simulations/fig2_oracle.csv};
\addplot[cgray] table[x=t,y=q,col sep=comma] {Simulations/fig2_naive.csv};
\addplot[cblue] table[x=t,y=q,col sep=comma] {Simulations/fig2_aci.csv};
\addplot[cred]  table[x=t,y=q,col sep=comma] {Simulations/fig2_sw.csv};
\legend{oracle,non-adapt,ACI,sliding-win}
\end{axis}
\end{tikzpicture}
\hfill
\begin{tikzpicture}
\begin{axis}[
    width=7.3cm, height=5.5cm,
    title={\small cumulative $R(T)$},
    grid=both, grid style={gray!20},
]
\addplot[cgray, thick] table[x=t,y=RT,col sep=comma] {Simulations/fig2_naive.csv};
\addplot[cblue, thick] table[x=t,y=RT,col sep=comma] {Simulations/fig2_aci.csv};
\addplot[cred,  thick] table[x=t,y=RT,col sep=comma] {Simulations/fig2_sw.csv};
\end{axis}
\end{tikzpicture}

\vspace{3mm}
\begin{tikzpicture}
\begin{axis}[
    width=7.3cm, height=5.5cm,
    title={\small 60-day rolling avg.\ set size},
    grid=both, grid style={gray!20},
]
\addplot[cgray] table[x=t,y=smooth_eff,col sep=comma] {Simulations/fig2_naive_smooth.csv};
\addplot[cblue] table[x=t,y=smooth_eff,col sep=comma] {Simulations/fig2_aci_smooth.csv};
\addplot[cred]  table[x=t,y=smooth_eff,col sep=comma] {Simulations/fig2_sw_smooth.csv};
\end{axis}
\end{tikzpicture}
\hfill
\begin{tikzpicture}
\begin{axis}[
    width=7.3cm, height=5.5cm,
    title={\small coverage-cost vs.\ efficiency},
    xlabel={mean set size},
    ylabel={mean violation $Q(T)/T$},
    grid=both, grid style={gray!20},
]
\addplot[only marks, mark=*, mark size=3pt, cgray] coordinates {(1.893,0.2255)};
\addplot[only marks, mark=*, mark size=3pt, cblue] coordinates {(1.758,0.2649)};
\addplot[only marks, mark=*, mark size=3pt, cred]  coordinates {(1.703,0.2842)};
\end{axis}
\end{tikzpicture}
\caption{Efficiency measured against the dynamic (centered-window) oracle $u_t^*$.}
\label{fig:2}
\end{figure}

\subsection{Figure 3: combined long-run cost}

\begin{figure}[h!]
\centering
\begin{tikzpicture}
\begin{axis}[
    width=13cm, height=7cm,
    xlabel={trading day},
    ylabel={$Q(T)+R(T)$},
    grid=both, grid style={gray!20},
    legend pos=north west, legend cell align=left,
]
\addplot[cgray, thick] table[x=t,y=combined,col sep=comma] {Simulations/fig3_naive.csv};
\addplot[cblue, thick] table[x=t,y=combined,col sep=comma] {Simulations/fig3_aci.csv};
\addplot[cred,  thick] table[x=t,y=combined,col sep=comma] {Simulations/fig3_sw.csv};
\legend{Non-adaptive, ACI (Model I), Sliding-window (Model II)}
\end{axis}
\end{tikzpicture}
\caption{Combined coverage-violation-plus-efficiency-regret cost, $Q(T)+R(T)$.}
\label{fig:3}
\end{figure}
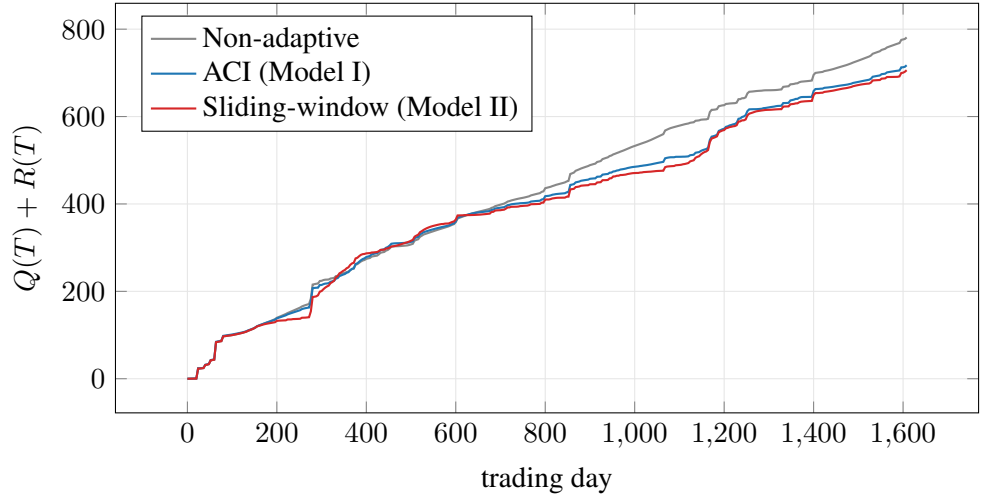

\section{Results table}
\label{sec:dax-results}

\begin{table}[h!]
\centering
\begin{tabular}{lccccc}
\toprule
Method & Emp.\ coverage & $Q(T)$ & $R(T)$ & mean set size & $Q(T)+R(T)$ \\
\midrule
Non-adaptive            & 0.926 & 363.4 & 417.7 & 1.893 & 781.1 \\
ACI (Model I)            & 0.909 & 426.6 & 290.8 & 1.758 & 717.4 \\
Sliding-window (Model II)& 0.896 & 457.5 & 248.5 & 1.703 & 706.1 \\
\bottomrule
\end{tabular}
\caption{Final-horizon summary statistics, all three methods on the identical score
sequence.}
\end{table}

\subsection{Interpretation}
In this section, we interpret the  results obtained so far in detail.
\label{sec:dax-interp}

\subsubsection{Coverage alone}
Judged only by the signed running error that \cite{gibbs2021adaptive} themselves
report (Figure~\ref{fig:1}, top), all three methods look essentially
interchangeable: each hovers near zero, none obviously fails. This is precisely the
blind spot the unified framework flags --- signed averaging lets long stretches of
one-directional miscoverage cancel against later over-coverage. The middle panel,
which sums only the \emph{one-sided} shortfall $(r_t-q_t)_+$ and never lets a good
day cancel a bad one, tells a different story: the non-adaptive method has the
\emph{lowest} cumulative violation of the three, with ACI in the middle and the
sliding-window tracker highest. Taken alone, this panel would suggest the simplest,
static method is ``best'' at coverage.

\subsubsection{Efficiency alone}
That conclusion reverses completely once efficiency is measured honestly against a
moving target. Figure~\ref{fig:2} (top row) shows the oracle threshold $u_t^*$
swinging between roughly $1.0$ and $3.7$ over the sample, tracking real regime
changes in market volatility; the non-adaptive threshold, by construction, cannot
move at all. Its cumulative efficiency regret $R(T)$ is therefore the largest of the
three ($417.7$, versus $290.8$ for ACI and $248.5$ for the sliding-window tracker),
and its raw mean prediction-set size is also the widest ($1.893$). In other words,
the non-adaptive method's apparent coverage advantage above was bought
entirely by carrying an oversized, non-responsive interval throughout the sample ---
exactly the ``trivial efficiency'' loophole that a coverage-only guarantee permits.
Between the two adaptive methods, the sliding-window tracker (which sees the full
recent distribution of scores each round) tracks $u_t^*$ more tightly than ACI
(which only ever sees one bit, covered or not, per round), consistent with the
paper's characterization of how performance should degrade as feedback coarsens.

\subsubsection{Coverage and efficiency together}
Figure~\ref{fig:3} adds the two cumulative costs, $Q(T)+R(T)$, and this is the
only view that scores every method on both axes simultaneously as time
progresses. By the end of the horizon the ranking is: sliding-window ($706.1$) $<$
ACI ($717.4$) $<$ non-adaptive ($781.1$). The non-adaptive method, the apparent
winner under coverage alone, is clearly the worst once its efficiency cost is
charged against it. The two adaptive methods are close, with the sliding-window
tracker slightly ahead overall: it pays a higher coverage-violation cost than ACI,
but more than makes up for it by tracking the true volatility target far more
closely, and the net trade favors it. The practical reading is that using richer
per-round feedback (the whole local score distribution, rather than a single
covered/not-covered bit) is worthwhile once both coverage and efficiency are priced
in together, even though it would look like the worse choice if either quantity were
inspected in isolation.

\subsection{Adversarial input to contrast  ACI with Window Tracking Algorithm}

\subsubsection{Adversarial conformity score: experiment setup}
\label{sec:adv-setup}

The goal is to construct an arbitrary (adversarial, non-stochastic) conformity-score
sequence $\{r_t\}$ that exercises the one structural difference between the two
algorithms: ACI (Model~I) uses only a single round of memory and requires no
distributional assumption, while the sliding-window quantile tracker (Model~II) is
built on an implicit assumption that the last $W$ scores are representative of the
current round, which is exactly what fails when the environment changes faster than
$W$.

\paragraph{Adversarial input.} A square wave of period $p$ alternating between a
``quiet'' level $0$ and a ``loud'' level $M$,
\[
  r_t \;=\; \begin{cases} 0 & \lfloor t/p \rfloor \text{ even} \\ M & \lfloor t/p \rfloor \text{ odd,} \end{cases}
  \qquad t = 1,\dots,T,
\]
with $M = 5$, $T = 4000$. This is a legitimate instance of Model~I's fully
adversarial setting: the sequence is deterministic and arbitrary, with no
distributional structure assumed or required.

\paragraph{Algorithms, run on the identical sequence, with fixed hyperparameters
across every $p$ (no per-instance retuning).}
\begin{itemize}
  \item \textbf{ACI (Model~I).}
        $q_{t+1} = \Pi_{[0,Q_{\max}]}\big[q_t + \gamma((1-z_t)-\alpha)\big]$,
        $z_t = \mathbb 1(r_t \le q_t)$, step size $\gamma = 0.20$.
  \item \textbf{Sliding-window quantile (Model~II).}
        $q_t = $ empirical $(1-\alpha)$-quantile of $\{r_{t-W},\dots,r_{t-1}\}$,
        window length $W = 150$.
\end{itemize}
Both use $\alpha = 0.10$, $Q_{\max}=2M$.

\paragraph{Metrics.} Since $\{r_t\}$ is deterministic, the Model~I benchmark is exact:
the dynamic oracle is $u_t^{*}=r_t$ itself, and with $\Lambda(q)=q$ (identity,
monotone),
\begin{align}
  \textbf{Coverage violation:}\quad & Q(T) = \sum_{t\le T} (r_t-q_t)_+, \\
  \textbf{Efficiency regret:}\quad  & R(T) = \sum_{t\le T} \big|\Lambda(q_t)-\Lambda(u_t^{*})\big| = \sum_{t\le T}|q_t-r_t|.
\end{align}
Both are computed as a function of the oscillation period $p$, swept over
$$p\in\{10,15,20,30,50,75,100,150,200,300,500,750,1000,1500\},$$ holding $W=150$ and
$\gamma=0.20$ fixed throughout, so that $p<W$, $p\approx W$, and $p\gg W$ are all
represented in the same sweep.

\paragraph{Results}
\label{sec:adv-figures}
Under this model, we plot the coverage: $Q(T)$ vs. oscillation period in Fig. \ref{fig:coverage} and efficiency: $R(T)$ vs. oscillation period in Fig. \ref{fig:efficiency}.

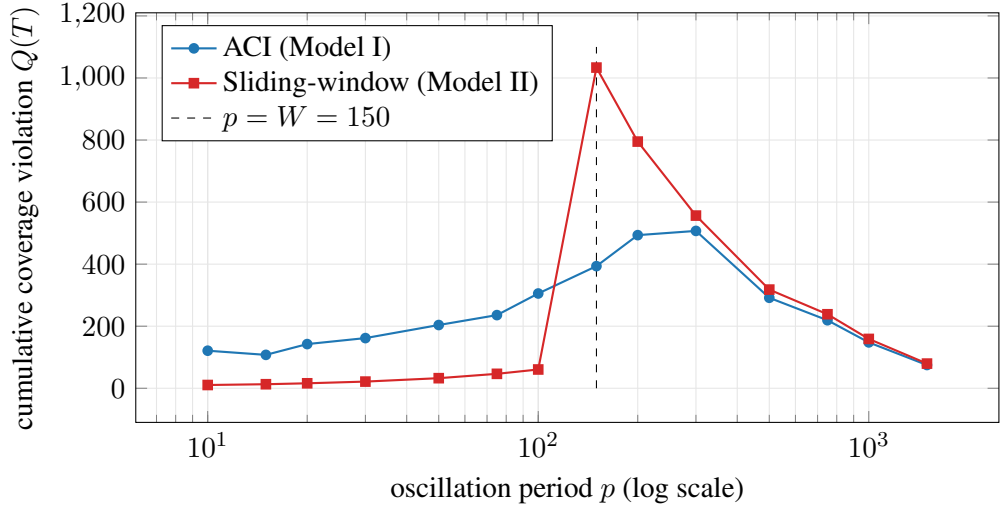
\begin{figure}[h!]
\centering
\begin{tikzpicture}
\begin{axis}[
    width=13cm, height=7cm,
    xmode=log,
    xlabel={oscillation period $p$ (log scale)},
    ylabel={cumulative coverage violation $Q(T)$},
    grid=both, grid style={gray!20},
    legend pos=north west, legend cell align=left,
]
\addplot[cblue, thick, mark=*, mark size=1.6pt] table[x=p,y=QTa,col sep=comma] {Simulations/sweep.csv};
\addplot[cred,  thick, mark=square*, mark size=1.6pt] table[x=p,y=QTw,col sep=comma] {Simulations/sweep.csv};
\addplot[black, dashed] coordinates {(150,0) (150,1100)};
\legend{ACI (Model I), Sliding-window (Model II), $p=W=150$}
\end{axis}
\end{tikzpicture}
\caption{Coverage violation $Q(T)$ as a function of oscillation period $p$, fixed
$W=150$ and $\gamma=0.20$.}
\label{fig:coverage}
\end{figure}

\begin{figure}[h!]
\centering
\begin{tikzpicture}
\begin{axis}[
    width=13cm, height=7cm,
    xmode=log,
    xlabel={oscillation period $p$ (log scale)},
    ylabel={cumulative efficiency regret $R(T)$},
    grid=both, grid style={gray!20},
    legend pos=north east, legend cell align=left,
]
\addplot[cblue, thick, mark=*, mark size=1.6pt] table[x=p,y=RTa,col sep=comma] {Simulations/sweep.csv};
\addplot[cred,  thick, mark=square*, mark size=1.6pt] table[x=p,y=RTw,col sep=comma] {Simulations/sweep.csv};
\addplot[black, dashed] coordinates {(150,0) (150,10200)};
\legend{ACI (Model I), Sliding-window (Model II), $p=W=150$}
\end{axis}
\end{tikzpicture}
\caption{Efficiency regret $R(T)$ (tracking error against the exact oracle
$u_t^*=r_t$) as a function of oscillation period $p$.}
\label{fig:efficiency}
\end{figure}
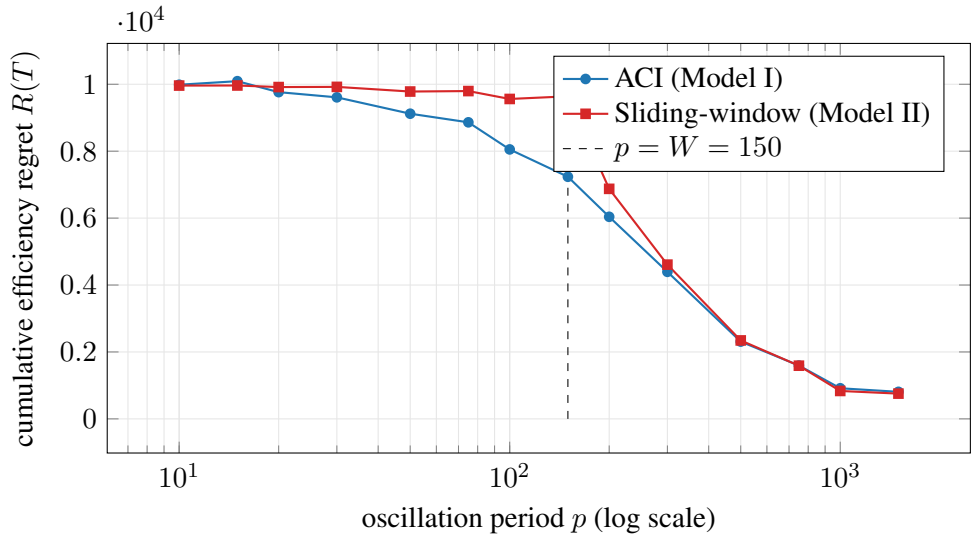

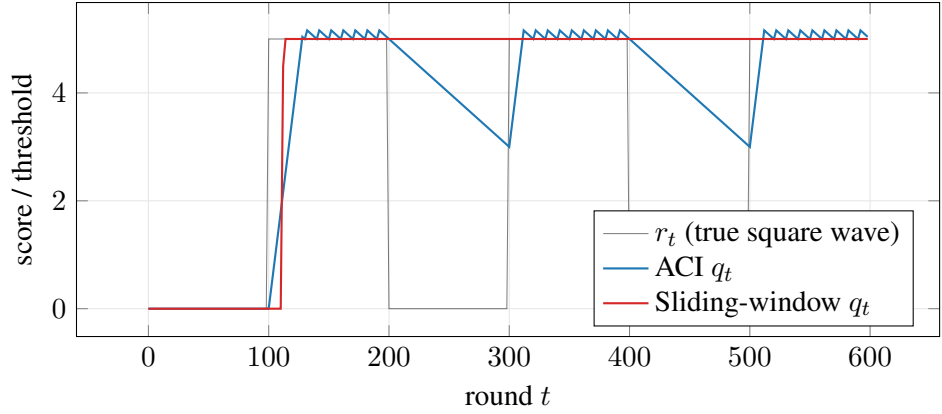
\begin{figure}[h!]
\centering
\begin{tikzpicture}
\begin{axis}[
    width=13cm, height=6cm,
    xlabel={round $t$},
    ylabel={score / threshold},
    grid=both, grid style={gray!20},
    legend pos=south east, legend cell align=left,
]
\addplot[black, thin, opacity=0.5] table[x=t,y=r,col sep=comma] {Simulations/trace.csv};
\addplot[cblue, thick] table[x=t,y=qa,col sep=comma] {Simulations/trace.csv};
\addplot[cred,  thick] table[x=t,y=qw,col sep=comma] {Simulations/trace.csv};
\legend{$r_t$ (true square wave), ACI $q_t$, Sliding-window $q_t$}
\end{axis}
\end{tikzpicture}
\caption{Threshold trajectories for $p=100 < W=150$, illustrating the mechanism
behind Figures~\ref{fig:coverage}--\ref{fig:efficiency}.}
\label{fig:trace}
\end{figure}

\begin{table}[h!]
\centering
\small
\begin{tabular}{rrrrrl}
\toprule
$p$ & $Q(T)$ ACI & $Q(T)$ window & $R(T)$ ACI & $R(T)$ window & combined winner \\
\midrule
10   & 121.0 & 10.5   & 9985.8 & 9960.5  & window \\
15   & 107.6 & 13.0   & 10091.9 & 9963.0 & window \\
20   & 142.4 & 16.0   & 9764.3 & 9916.0  & ACI \\
50   & 203.8 & 32.5   & 9119.9 & 9782.5  & ACI \\
100  & 305.4 & 60.5   & 8052.6 & 9560.5  & ACI \\
150  & 393.6 & 1033.5 & 7235.0 & 9639.5  & ACI \\
200  & 493.5 & 795.0  & 6039.0 & 6874.5  & ACI \\
500  & 291.6 & 318.0  & 2309.7 & 2344.5  & ACI \\
1000 & 147.8 & 159.0  & 915.1  & 834.5   & window \\
1500 & 74.9  & 79.5   & 808.3  & 755.0   & window \\
\bottomrule
\end{tabular}
\caption{Selected rows of the full sweep; ``combined winner'' is by $Q(T)+R(T)$.}
\end{table}

\subsubsection{Interpretation}\label{sec:adv-interp}
In this subsection, we interpret the  results obtained so far in detail for the arbitrary input.

\subsubsection{Coverage alone (Figure~\ref{fig:coverage})}
The window tracker has \emph{lower} coverage violation than ACI everywhere except
right at $p=W=150$, where it spikes catastrophically to $1033.5$ -- more than double
its neighbors at $p=100$ and $p=200$. Away from that spike, the window's advantage
on this axis alone is not a sign of genuinely better calibration: as
Section~\ref{sec:adv-interp} (below) shows, it comes from the window sitting near
the loud level $M$ almost permanently once $p<W$, which trivially avoids violating
coverage but says nothing about whether the resulting intervals are useful. The one
place the window's coverage genuinely breaks down is the resonance point
$p\approx W$, where the window is, for a sustained stretch, aligned so that it is
transitioning between regimes at the worst possible phase.

\subsubsection{Efficiency alone (Figure~\ref{fig:efficiency})}
This is where the two algorithms separate clearly. For every $p$ from about $20$ up
through $750$ -- i.e.\ from far below $W$ to several multiples above it -- ACI's
efficiency regret is lower than the window's, often substantially so (e.g.\ $8052.6$
vs.\ $9560.5$ at $p=100$; $7235.0$ vs.\ $9639.5$ at $p=150$). Only at the extremes,
$p\le 15$ (oscillation far faster than either algorithm can plausibly track) and
$p\ge 1000$ (oscillation far slower than $W$, so the window's memory is no longer a
liability), does the window recover a small edge. The mechanism is directly visible
in Figure~\ref{fig:trace}: ACI's threshold rises and falls with the square wave,
lagged but genuinely responsive, while the window's threshold is comparatively flat
-- once $p<W$, every window contains a mixture of both regimes, so its quantile is
structurally unable to commit to either the low or the high state.

\subsubsection{Coverage and efficiency together}
Read jointly, the two figures tell a three-regime story rather than a blanket
``ACI is better'' or ``window is better'' claim.
\begin{itemize}
  \item \textbf{$p\lesssim 15$} (oscillation much faster than $W$): the window wins
        narrowly, by defaulting near the loud level and rarely paying a coverage
        cost for it -- the same trivial-efficiency trick a static predictor uses,
        not a real advantage.
  \item \textbf{$20 \lesssim p \lesssim 750$} (oscillation comparable to, or a few
        multiples of, $W$): ACI wins clearly on the combined cost, and the
        window's fixed-length memory is a genuine liability here -- most
        dramatically at $p\approx W$, where it is actively unstable.
  \item \textbf{$p\gtrsim 1000$} (oscillation much slower than $W$): the window's
        implicit assumption -- that the recent past looks like the present -- is
        now accurate, and it edges back ahead, consistent with it being the
        rate-optimal choice in the slowly-varying stochastic regime the theory
        was built for.
\end{itemize}
The practically useful takeaway is that there is a concrete, non-exotic middle band
of adversarial inputs -- one whose only defining feature is oscillating on a
timescale comparable to the window's own memory length -- where a purely
single-round-reactive algorithm with no distributional requirement is structurally
preferable to a window-based estimator, and this is exactly the regime the
adversarial (Model~I) guarantee is designed to cover with no assumptions at all.

\subsection{Regret of ACI as a function of path length ${\cal P}_T$}

Recall that we showed in Theorem \ref{thm:adversarial} the regret of ACI algorithm is \[
  R(T) = O\!\big(\sqrt{T(1+{\cal P}_T)}\big).
\]
In this setup, we want to test its tightness.

\paragraph{Simulation setting.} Similar to Section \ref{sec:adv-setup}, the conformity-score sequence is again a square
wave alternating between $0$ and $M=5$, but parameterized directly by the number
of segments $K$ rather than by oscillation period, so that the path length is
exact and controlled:
\[
  {\cal P}_T = (K-1)M.
\]
This decouples ``how much the sequence moves'' (${\cal P}_T$) from ``how often''
(period), which the period-based parameterization used earlier conflated. The
step size is set to the Theorem \ref{thm:adversarial}'s prescribed scaling
\[
  \eta = \min\!\Big(\sqrt{\tfrac{1+{\cal P}_T}{T}},\ 0.9\Big),
\]
recomputed for each instance from the known ${\cal P}_T$, Theorem \ref{thm:adversarial}'s guarantee only
holds under this specific tuning, so a rate check with a single fixed step size
(as used in every earlier experiment) would not actually test the theorem's
claim.

Two independent sweeps isolate the bound's two variables:
\begin{itemize}
  \item \textbf{Sweep 1.} Fix $T=6000$, vary ${\cal P}_T$ via $K\in\{2,\dots,80\}$.
        Theory predicts $R(T)\propto {\cal P}_T^{0.5}$ at large ${\cal P}_T$.
  \item \textbf{Sweep 2.} Fix ${\cal P}_T=25$ (only $K=6$ segments, regardless of
        horizon), vary $T$ from $500$ to $38{,}000$. Theory predicts
        $R(T)\propto T^{0.5}$.
\end{itemize}

\paragraph{Why this input design makes sense.} The square-wave sweep answered whether
ACI has a structural weakness the window tracker doesn't share (it doesn't); it
never validated ACI's own claimed rate, since period and path length were
entangled there and the step size was fixed rather than theorem-scaled. Fixing
${\cal P}_T$ directly, fixing $T$ directly, and setting $\eta$ the way the proof
requires isolates exactly the quantity the theorem makes a claim about.

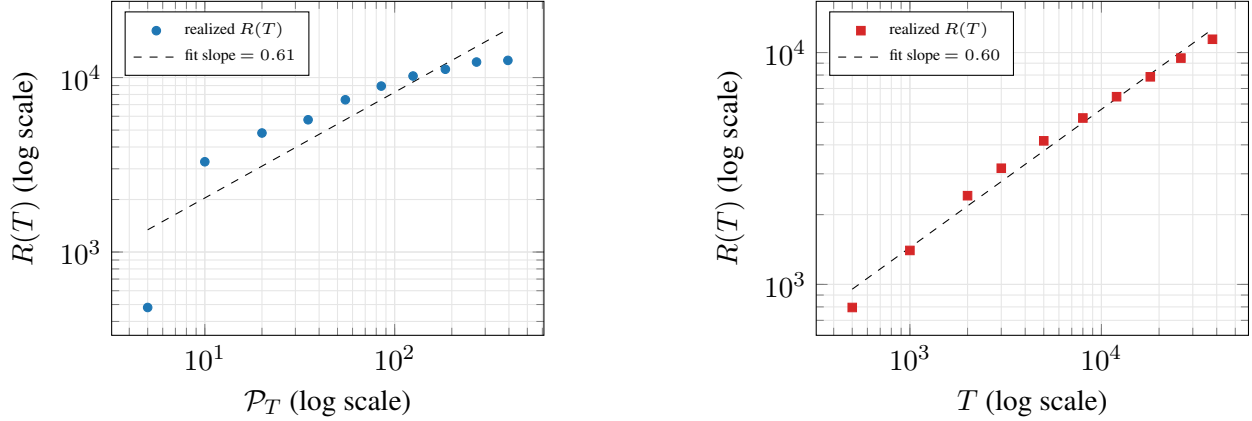
\begin{figure}[h!]
\centering
\begin{tikzpicture}
\begin{loglogaxis}[
    width=7.3cm, height=6cm,
    xlabel={${\cal P}_T$ (log scale)},
    ylabel={$R(T)$ (log scale)},
    grid=both, grid style={gray!20},
    legend pos=north west, legend cell align=left, legend style={font=\tiny},
]
\addplot[cblue, only marks, mark=*, mark size=1.6pt] table[x=PT,y=RT,col sep=comma] {Simulations/expA_sweep1.csv};
\addplot[black, dashed] table[x=PT,y=RT,col sep=comma] {Simulations/expA_fit1.csv};
\legend{realized $R(T)$, fit slope $=0.61$}
\end{loglogaxis}
\end{tikzpicture}
\hfill
\begin{tikzpicture}
\begin{loglogaxis}[
    width=7.3cm, height=6cm,
    xlabel={$T$ (log scale)},
    ylabel={$R(T)$ (log scale)},
    grid=both, grid style={gray!20},
    legend pos=north west, legend cell align=left, legend style={font=\tiny},
]
\addplot[cred, only marks, mark=square*, mark size=1.6pt] table[x=T,y=RT,col sep=comma] {Simulations/expA_sweep2.csv};
\addplot[black, dashed] table[x=T,y=RT,col sep=comma] {Simulations/expA_fit2.csv};
\legend{realized $R(T)$, fit slope $=0.60$}
\end{loglogaxis}
\end{tikzpicture}
\caption{Log-log rate check for ACI (Model~I). Left: fixed $T=6000$, path length
${\cal P}_T$ varied. Right: fixed ${\cal P}_T=25$, horizon $T$ varied. Theory predicts slope
$0.5$ in both.}
\label{fig:rate-check}
\end{figure}

\paragraph{Output.}
\[
  \text{Sweep 1 (vary } {\cal P}_T\text{): slope} = 0.606
  \qquad
  \text{Sweep 2 (vary } T\text{): slope} = 0.596.
\]

\paragraph{Interpretation.} Theory predicts an exponent of $0.5$ in each
variable; the realized exponents are modestly but consistently above that,
around $0.6$ in both independent checks. Three points are worth making explicit.

\begin{enumerate}
  \item \emph{This is not a violation of the theorem.} The theorem states an
        upper bound, $R(T)=O(\sqrt{T(1+{\cal P}_T)})$. A realized exponent above $0.5$
        over a finite range does not contradict an $O(\cdot)$ statement, since
        constants and lower-order terms can dominate at the horizon lengths
        tested. What it would contradict is a matching \emph{lower} bound
        asserting the rate is exactly $\Theta(\sqrt{T(1+{\cal P}_T)})$, and the
        paper does not claim such a lower bound for Model~I (unlike Model~II,
        which has a matching lower bound, Theorem~7). A faster empirical
        growth than $0.5$ here is not evidence against anything the paper
        actually asserts.
  \item \emph{Finite-range curvature is a likely explanation.} Log-log slopes
        estimated over roughly one to two orders of magnitude are sensitive to
        lower-order terms: a function of the form $a\sqrt{T(1+{\cal P}_T)}+b$ has a
        log-log slope systematically biased \emph{above} $0.5$ at
        small-to-moderate scale, converging to $0.5$ only once the leading
        term dominates. Neither sweep spans more than about $1.5$ decades, so
        both are plausibly still in this transient regime.
  \item \emph{Saturation at large ${\cal P}_T$.} An earlier, coarser version of this
        experiment showed $R(T)$ plateauing once oscillation became extremely
        fast, because the achievable tracking error is capped by $M$ itself.
        The largest-$K$ points in Sweep~1 may already be drifting toward that
        regime, which would locally steepen the log-log slope above the true
        asymptotic value.
\end{enumerate}

\emph{Bottom line.} The exponents recovered ($0.60$-$0.61$ vs.\ a theoretical
$0.50$) are close enough to support the qualitative claim that the regret is  sublinear, roughly
square-root growth in both $T$ and ${\cal P}_T$ separately, without being tight
enough, or spanning a wide enough dynamic range, to confirm the theorem's exact
exponent. A convincing confirmation would need several more decades of $T$ and
${\cal P}_T$, checking specifically whether the slope drifts \emph{toward} $0.5$ as the
range grows, which would distinguish ``asymptotically tight, just not there
yet'' from ``the realized rate is genuinely faster than the bound in this
regime.''

\subsection{Covariate-Dependent Coverage (Model III / HP-ACI)}

\subsubsection{Simulation setup}
\label{sec:m3-setup}

Models I and II track a single scalar threshold and give only \emph{marginal}
coverage guarantees. Model III addresses the harder, more realistic setting where
the conformity-score distribution depends on an observed covariate $x$, so a single
global threshold is necessarily miscalibrated for some parts of the covariate space
even when it is perfectly calibrated on average. This experiment simulates that
setting directly and compares the paper's H\"older-Partitioned ACI (HP-ACI,
Section~4.2 of the main text) against a single global (unpartitioned) ACI instance
run on the identical data.

\paragraph{Covariate and oracle threshold.} $X_t \sim \mathrm{Uniform}[0,1]$ i.i.d.\
each round ($d=1$). The oracle threshold function drifts smoothly in both space and
time,
\[
  u_t^*(x) \;=\; 2 + 1.5\sin\!\big(2\pi(x - t/T_{\text{period}})\big), \qquad T_{\text{period}}=3000,
\]
which is Lipschitz (hence $\beta=1$-H\"older) in $x$ for every $t$, satisfying
Assumption~C1, and gives a finite, computable path-length budget $S_T$.

\paragraph{Conditional score model.} Given $X_t=x$, the conformity score is
\[
  r_t \;=\; u_t^*(x) + c\big(Z_t - (1-\alpha)\big), \qquad Z_t\sim\mathrm{Uniform}(0,1),\ c=1,
\]
clipped to $\mathcal A=[0,6]$. This gives $F_t(q\mid x) = \mathrm{clip}\big((q-u_t^*(x))/c + (1-\alpha),\,0,\,1\big)$,
a uniform conditional density $f_t(q\mid x)=1/c$ on the support, so
$f^s_{\min}=f^s_{\max}=1/c=1$, and by construction $F_t(u_t^*(x)\mid x) = 1-\alpha$
exactly, i.e.\ $u_t^*$ is the true oracle threshold at every round.

\paragraph{Algorithms, run on the identical realized scores.}
\begin{itemize}
  \item \textbf{HP-ACI (Model~III).} Partition $[0,1]$ into $N=\lceil 1/h\rceil$
        cells of width $h=0.1$; each cell maintains its own ACI instance, updated
        only when $X_t$ falls in that cell, with step size $\eta=0.25$ (within the
        stability bound $\eta\le f^s_{\min}/(2(f^s_{\max})^2)=0.5$ required by
        Assumption~C5).
  \item \textbf{Global ACI (no partition).} A single scalar threshold, standard
        ACI update, step size $\gamma=0.05$, blind to $X_t$.
\end{itemize}
$\alpha=0.10$ throughout. $S_T\approx 18.8$ (computed directly from the true oracle
drift).

\paragraph{Metrics.} As in Model~III's definitions (Section~1.7 of the main text),
with $\Lambda(q)=q$:
\begin{align}
  \textbf{Coverage violation:}\quad & Q(T) = \sum_{t\le T}\big(1-\alpha - F_t(q_t(X_t)\mid X_t)\big)_+, \\
  \textbf{Efficiency regret:}\quad  & R(T) = \sum_{t\le T}\big|q_t(X_t) - u_t^*(X_t)\big|.
\end{align}
Both are computed exactly, using the known closed-form $F_t(\cdot\mid x)$. In
addition, we report \emph{conditional} coverage, binning rounds by covariate $x$
(and, in the heatmap, jointly by $x$ and a block of time), since the whole point of
Model~III is calibration \emph{as a function of $x$}, which a single marginal
number cannot reveal.

\begin{figure}[h!]
\centering
\begin{tikzpicture}
\begin{axis}[
    width=6.3cm, height=5.5cm,
    xlabel={covariate $x$}, ylabel={empirical conditional coverage},
    grid=both, grid style={gray!20},
    legend style={font=\tiny}, legend pos=south east,
]
\addplot[cblue, thick, mark=*, mark size=1.4pt] table[x=x,y=cov_hp,col sep=comma] {Simulations/m3_covbybin.csv};
\addplot[cred,  thick, mark=square*, mark size=1.4pt] table[x=x,y=cov_global,col sep=comma] {Simulations/m3_covbybin.csv};
\addplot[black, dashed] coordinates {(0,0.9) (1,0.9)};
\legend{HP-ACI, Global ACI, target}
\end{axis}
\end{tikzpicture}
\hfill
\begin{tikzpicture}
\begin{axis}[
    width=6.3cm, height=5.5cm,
    xlabel={covariate $x$}, ylabel={threshold},
    grid=both, grid style={gray!20},
    legend style={font=\tiny}, legend pos=south east,
]
\addplot[black, dashed, thick] table[x=x,y=oracle,col sep=comma] {Simulations/m3_surface.csv};
\addplot[cblue, thick] table[x=x,y=hp,col sep=comma] {Simulations/m3_surface.csv};
\addplot[cred,  thick] table[x=x,y=global,col sep=comma] {Simulations/m3_surface.csv};
\legend{oracle,HP-ACI,global}
\end{axis}
\end{tikzpicture}
\hfill
\begin{tikzpicture}
\begin{axis}[
    width=6.3cm, height=5.5cm,
    xlabel={round $t$}, ylabel={cumulative $Q(T)$},
    grid=both, grid style={gray!20},
    legend style={font=\tiny}, legend pos=north west,
]
\addplot[cblue, thick] table[x=t,y=QT_hp,col sep=comma] {Simulations/m3_QT.csv};
\addplot[cred,  thick] table[x=t,y=QT_global,col sep=comma] {Simulations/m3_QT.csv};
\legend{HP-ACI,Global ACI}
\end{axis}
\end{tikzpicture}
\caption{Left: conditional coverage by covariate bin, marginalized over the whole
horizon. Middle: threshold surface $q_t(x)$ vs.\ the true oracle $u_t^*(x)$ at a
snapshot round $t=4500$. Right: cumulative coverage violation $Q(T)$ over time.}
\label{fig:m3-summary}
\end{figure}
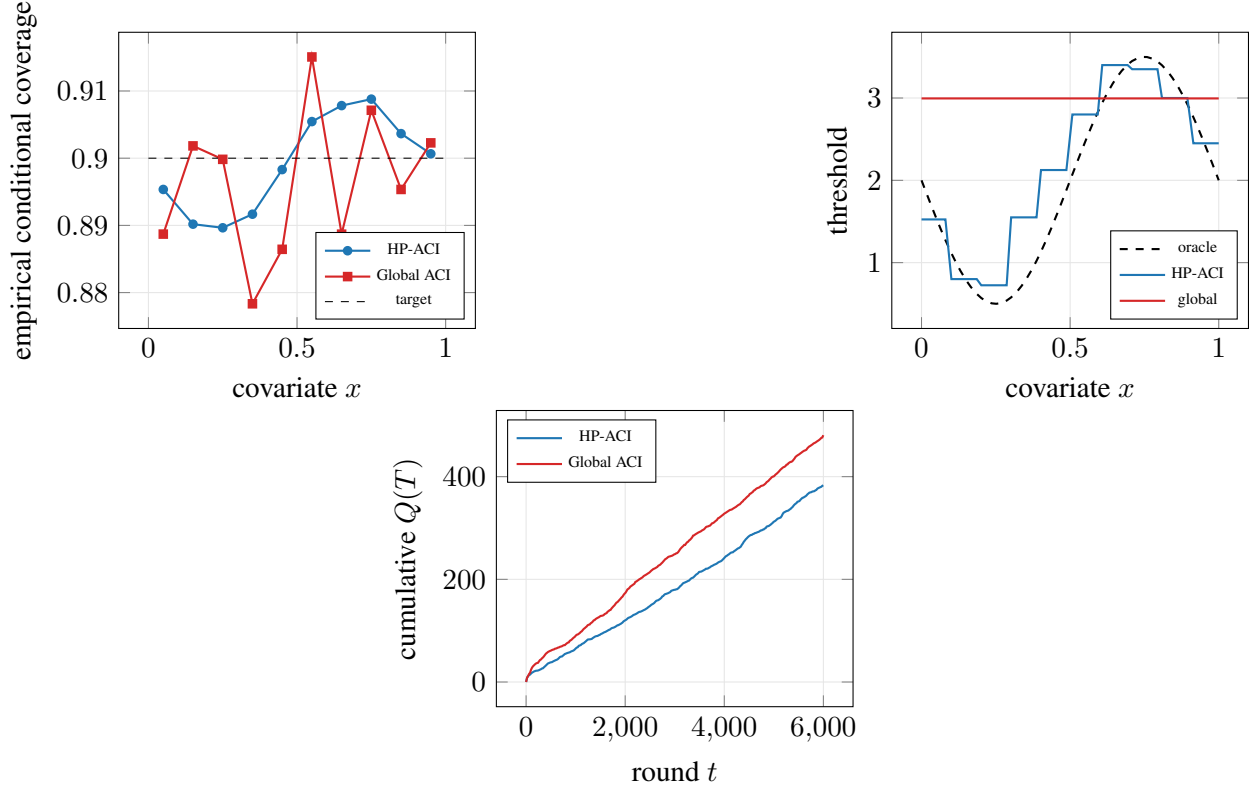

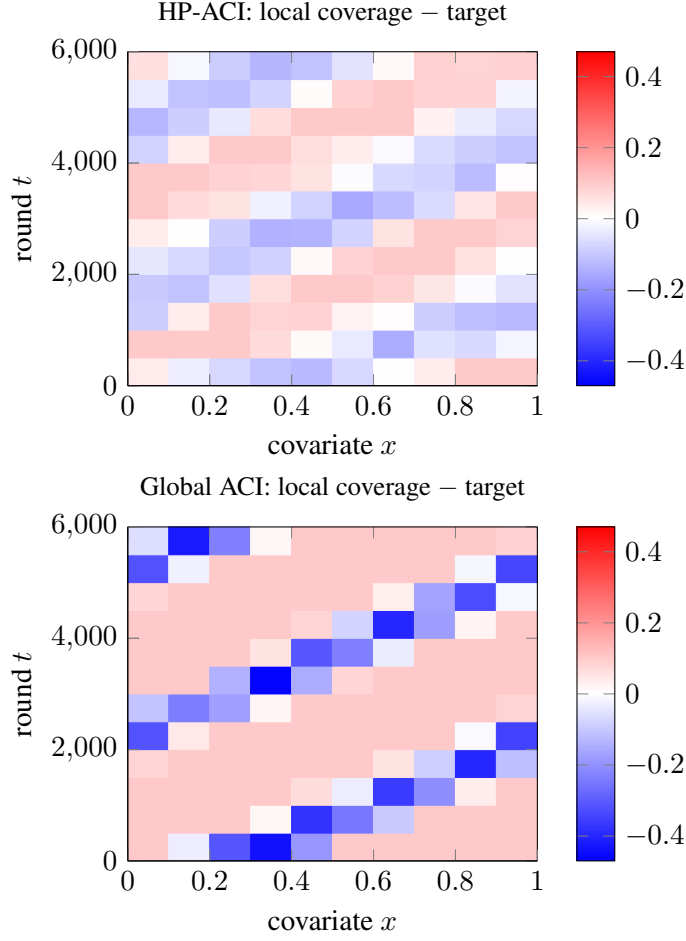
\begin{figure}[h!]
\centering
\pgfplotsset{
  colormap={rdbu}{rgb=(0,0,1) rgb=(1,1,1) rgb=(1,0,0)},
}
\begin{tikzpicture}
\begin{axis}[
    width=7cm, height=6cm,
    xlabel={covariate $x$}, ylabel={round $t$},
    title={\small HP-ACI: local coverage $-$ target},
    colorbar, point meta min=-0.47, point meta max=0.47,
    view={0}{90},
]
\addplot3[
    matrix plot*, mesh/rows=12, mesh/cols=10, point meta=explicit,
    shader=flat corner,
] table[x=x,y=t,meta=val,col sep=comma] {Simulations/m3_heat_hp.csv};
\end{axis}
\end{tikzpicture}
\hfill
\begin{tikzpicture}
\begin{axis}[
    width=7cm, height=6cm,
    xlabel={covariate $x$}, ylabel={round $t$},
    title={\small Global ACI: local coverage $-$ target},
    colorbar, point meta min=-0.47, point meta max=0.47,
    view={0}{90},
]
\addplot3[
    matrix plot*, mesh/rows=12, mesh/cols=10, point meta=explicit,
    shader=flat corner,
] table[x=x,y=t,meta=val,col sep=comma] {Simulations/m3_heat_global.csv};
\end{axis}
\end{tikzpicture}
\caption{Local (windowed) conditional coverage deviation from target,
jointly over covariate $x$ (horizontal) and time block (vertical). Blue = under
target, red = over target. Global ACI shows a persistent, large-amplitude
$x$-dependent banding pattern; HP-ACI's deviations are smaller and less
spatially structured.}
\label{fig:m3-heatmap}
\end{figure}

\begin{table}[h!]
\centering
\begin{tabular}{lcccc}
\toprule
Method & $Q(T)$ & $R(T)$ & max local coverage dev. & std.\ of local coverage dev. \\
\midrule
HP-ACI (Model III)  & 383.1 & 1882.0 & 0.161 & 0.082 \\
Global ACI (no partition) & 480.8 & 7227.5 & 0.464 & 0.162 \\
\bottomrule
\end{tabular}
\caption{Final-horizon statistics. ``Local coverage deviation'' is computed on the
$12\times 10$ (time-block $\times$ covariate-bin) grid underlying
Figure~\ref{fig:m3-heatmap}.}
\end{table}

\subsection{Interpretation}
\label{sec:m3-interp}

\subsubsection{Marginal coverage is a poor diagnostic here}
Averaged over the whole horizon and the whole covariate space, both methods land
close to the $90\%$ target ($89.9\%$ for HP-ACI, $89.7\%$ for global ACI). Judged
only by that number, the two algorithms look interchangeable -- which is exactly
the failure mode Model~III is designed to catch. Global ACI's marginal number looks
fine \emph{because} its errors at different $x$ average out across a full horizon
that spans several periods of the oracle's oscillation, not because it is actually
calibrated at any given $(x,t)$.

\subsection{Conditional coverage separates the methods clearly}
The heatmap (Figure~\ref{fig:m3-heatmap}) makes the difference obvious: global ACI
shows a persistent banding pattern, over-covering wherever the oracle threshold
happens to be locally low and under-covering wherever it is locally high, with
deviations from the $90\%$ target reaching $\pm 46\%$ in the worst cell and a
standard deviation of $0.162$ across the grid. HP-ACI's deviations are smaller in
both magnitude (max $0.161$, well under half the global method's worst case) and
structure (std.\ $0.082$, roughly half global ACI's). The mechanism is directly
visible in the threshold-surface panel of Figure~\ref{fig:m3-summary}: HP-ACI's
piecewise-constant surface tracks the shape of the sinusoidal oracle reasonably
well, while global ACI can only ever offer a single flat line, which is by
construction unable to match a threshold that genuinely depends on $x$.

\subsection{Coverage and efficiency together}
The combined picture in Table~2 is unambiguous, unlike the closer contests in
Models~I and~II: HP-ACI dominates global ACI on \emph{both} axes simultaneously
($Q(T)=383.1$ vs.\ $480.8$; $R(T)=1882.0$ vs.\ $7227.5$, roughly a $3.8\times$
efficiency-regret gap). This is a qualitatively different result from the
Model~I/II comparisons, where the two algorithms traded off against each other
depending on the regime. Here, once the ground truth genuinely depends on the
covariate, a method that ignores the covariate entirely is not offering a
different tradeoff -- it is simply unable to represent the right answer, at any
setting of its one free parameter, and partitioning is not an optional
refinement but a structural necessity.

\clearpage

\section{Conclusions} We studied online conformal prediction through the lens of online optimization, identifying prediction-set efficiency as an objective that should be optimized jointly with coverage rather than treated as a secondary consideration. While the literature has largely focused on maintaining coverage under distribution shift, theoretical guarantees that simultaneously control both coverage and prediction-set efficiency have remained essentially unavailable. Our framework addresses this gap by establishing sublinear guarantees for both quantities against appropriate dynamic benchmarks across adversarial, stochastic, and covariate-dependent settings. The results further replace traditional signed calibration measures, which permit cancellation of under- and over-coverage, with feasibility-based coverage guarantees that directly quantify violations of the desired coverage constraint. Collectively, these results provide a unified optimization framework for online conformal prediction that characterizes the fundamental trade-offs between statistical assumptions, available feedback, and achievable performance. More broadly, they suggest that future adaptive conformal methods should be evaluated not only by their ability to maintain coverage, but also by how efficiently they achieve it in non-stationary environments.
\bibliographystyle{plain} 
\bibliography{Refs}      
\section{Proof of Theorem~\ref{thm:scalar-upper}} \label{sec:ProofModel2}

The proof decomposes the total error into three sources: estimation noise in the empirical CDF, bias from distribution drift within the window, and the discrete jump size of the empirical CDF. We bound each component and optimize the window size $W$ to minimize their sum.

\subsection{Step 1: Uniform Concentration}

Define the empirical process deviation at time $t$:
\[
\Delta_{t,W} := \sup_{q\in A}|\widehat{F}_{t-1,W}(q)-\bar{F}_{t-1,W}(q)|.
\]
By Assumption~A3, its expectation is bounded by:
\begin{equation}\label{eq:concentration}
\E[\Delta_{t,W}] \le \frac{C_{VC}}{\sqrt{W}}.
\end{equation}

\subsection{Step 2: Quantile-Crossing Lemma}

\begin{lemma}[Quantile-crossing lemma]
\label{lem:qcl}
Let $\mathcal{E}_t := \sup_q|\bar{F}_{t-1,W}(q)-F_t(q)|$ denote the window-to-current drift. For any $t>W$, the algorithm's threshold $q_t$ satisfies:
\[
|F_t(q_t)-(1-\alpha)| \le \Delta_{t,W} + \mathcal{E}_t + \frac{1}{W}.
\]
\end{lemma}

\begin{proof}
We bound $F_t(q_t)$ from both sides.

\emph{Lower bound:} By definition, $q_t$ is the infimum achieving $\widehat{F}_{t-1,W}(q_t) \ge 1-\alpha$. Thus, $\bar{F}_{t-1,W}(q_t) \ge (1-\alpha) - \Delta_{t,W}$. Applying the drift bound $|\bar{F}_{t-1,W}(q_t)-F_t(q_t)| \le \mathcal{E}_t$ yields $F_t(q_t) \ge (1-\alpha) - \Delta_{t,W} - \mathcal{E}_t$.

\emph{Upper bound:} Because $\widehat{F}_{t-1,W}$ is a step function with jumps of size exactly $1/W$, the left limit satisfies $\widehat{F}_{t-1,W}(q_t^-) < 1-\alpha$. Consequently, the value at $q_t$ overshoots $1-\alpha$ by at most one jump: $\widehat{F}_{t-1,W}(q_t) \le (1-\alpha) + 1/W$. Chaining the deviations gives $F_t(q_t) \le \bar{F}_{t-1,W}(q_t) + \mathcal{E}_t \le (1-\alpha) + \Delta_{t,W} + \mathcal{E}_t + 1/W$. 

Combining both inequalities completes the proof.
\end{proof}

\subsection{Step 3: Converting to Threshold-Space Error}

By Assumption~A2, $F_t$ has density bounded below by $f_{\min} > 0$. By the Mean Value Theorem, there exists $\tilde{q}$ between $q_t$ and the oracle $u_t^*$ such that $|F_t(q_t)-F_t(u_t^*)| = f_t(\tilde{q})|q_t-u_t^*| \ge f_{\min}|q_t-u_t^*|$. Because $F_t(u_t^*)=1-\alpha$, we have $|q_t-u_t^*| \le f_{\min}^{-1} |F_t(q_t)-(1-\alpha)|$. Taking expectations and applying Lemma~\ref{lem:qcl} and \eqref{eq:concentration} gives the per-round tracking error:
\begin{equation}\label{eq:action-bound}
\E|q_t-u_t^*| \le \frac{1}{f_{\min}}\!\left(\frac{C_{VC}}{\sqrt{W}} + \mathcal{E}_t + \frac{1}{W}\right).
\end{equation}

\subsection{Step 4: Drift Telescoping}

\begin{lemma}[Cumulative drift bound]
\label{lem:drift}
$\displaystyle\sum_{t=W+1}^T\mathcal{E}_t\le W V_T.$
\end{lemma}

\begin{proof}
For any fixed $t>W$, $\tau\in[t-W,t-1]$, and $q\in A$, telescoping gives $|F_\tau(q)-F_t(q)| \le \sum_{k=\tau}^{t-1}\sup_{q'}|F_{k+1}(q')-F_k(q')|$. Averaging over the $W$ points in the window yields:
\[
\mathcal{E}_t \le \frac{1}{W}\sum_{\tau=t-W}^{t-1}\sum_{k=\tau}^{t-1}\sup_{q'}|F_{k+1}(q')-F_k(q')|.
\]
Summing from $t=W+1$ to $T$ and exchanging the order of summation, each index $k$ appears in the inner sum for at most $W$ values of $t$. Thus, the total drift is bounded by $W \sum_{k=1}^{T-1}\sup_{q'}|F_{k+1}(q')-F_k(q')| = W V_T$.
\end{proof}

\subsection{Step 5: Bounding $R(T)$ and $Q(T)$}

\paragraph{Coverage Violation $Q(T)$.}
For the first $W$ rounds, the error is trivially bounded by $1$. For $t>W$, taking the expectation of Lemma~\ref{lem:qcl} and substituting \eqref{eq:concentration} bounds the per-round expected error. Summing over $T$ and applying Lemma~\ref{lem:drift} to bound the cumulative drift yields:
\[
Q(T) =
\sum_{t=1}^T
\mathbb{E}\!\left[\left|F_t(q_t)-(1-\alpha)\right|\right]  \le W + \sum_{t=W+1}^T\!\left(\frac{C_{VC}}{\sqrt{W}}+\mathcal{E}_t+\frac{1}{W}\right) \le W + \frac{C_{VC}T}{\sqrt{W}} + W V_T + \frac{T}{W}.
\]

\paragraph{Dynamic Regret $R(T)$.}
Recall that $\Lambda$ is $L$-Lipschitz, thus 
$$R(T) \le L \sum_{t=1}^T\ \E|q_t-u_t^*|.$$
Thus, summing \eqref{eq:action-bound} and applying Lemma~\ref{lem:drift} yields:  
\[
R(T) \le L \left(W + \frac{1}{f_{\min}}\!\left(\frac{C_{VC}T}{\sqrt{W}} + W V_T + \frac{T}{W}\right)\right).
\]

\paragraph{Balancing $W$.}
To jointly optimize the leading parameters, we fold the initialization cost $W$ into the dominant terms to form the relaxed objective $\widetilde{H}(W) = W(1+V_T) + C_{VC}T / \sqrt{W}$. Setting the derivative $\widetilde{H}'(W)=0$ gives:
\[
(1+V_T) - \frac{C_{VC}T}{2W^{3/2}} = 0 \implies W^\dagger = \left(\frac{C_{VC}T}{2(1+V_T)}\right)^{2/3}.
\]
At $W = \lceil W^\dagger \rceil$, the primary objective evaluates to 
$$\widetilde{H}(W^\dagger) = 3\cdot 2^{-2/3}\left(1+V_T\right)^{1/3}\!\left(C_{VC}T\right)^{2/3} = O\!\left(T^{2/3}(1+V_T)^{1/3}\right).$$

The residual term $T/W^\dagger$ is evaluated as:
\[
\frac{T}{W^\dagger} = T^{1/3}\left(\frac{2(1+V_T)}{C_{VC}}\right)^{2/3} = \left(\frac{2}{C_{VC}}\right)^{2/3} \left(\frac{1+V_T}{T}\right)^{1/3} T^{2/3}(1+V_T)^{1/3}.
\]
Because the cumulative drift $V_T$ is structurally bounded by $T$, the ratio $(1+V_T)/T = O(1)$. Thus, the residual term is also strictly bounded by $O\!\left(T^{2/3}(1+V_T)^{1/3}\right)$ and does not compromise the rate. 

Substituting these bounds into $Q(T)$ and $R(T)$ yields
\[
Q(T) = O\!\left(T^{2/3}(1+V_T)^{1/3}\right), \quad R(T) = O\!\left(\left(\frac{T}{1+V_T}\right)^{2/3} + \frac{T^{2/3}(1+V_T)^{1/3}}{f_{\min}}\right). \qquad\blacksquare
\]

In the next section, we establish a minimax lower bound demonstrating that this upper bound is fundamentally tight.

\section{Adaptive Window Selection: Removing the Need to Know $V_T$}
\label{sec:hedge_windows}

{\bf Construction: Hedge over pinball loss.}
Fix the geometric grid $\mathcal{W}=\{2^i:i=0,1,\dots,\lceil\log_2T\rceil\}$, $\vert{}\mathcal{W}\vert{}=O(\log T)$. Run, in parallel, one sliding-window  algorithm instance per $W\in\mathcal{W}$: for $t\le W$ set $q_t^{(W)}:=0$ (any fixed point of $A$; matches the base algorithm's own initialization rule), and for $t>W$ let $q_t^{(W)}$ be the sliding-window threshold.  Maintain a Hedge distribution $w_t=(w_{t,W})_{W\in\mathcal{W}}$, updated each round on the pinball loss $\psi_t^{(W)}:=\psi_t(q_t^{(W)})$ (Definition~\ref{defn:pinball}), and play the Hedge-weighted average threshold
\[\bar q_t=\sum_{W\in\mathcal{W}}w_{t,W}\,q_t^{(W)}.\]

\begin{theorem}[Guarantee for Construction A]
\label{thm:constrA}
Under Assumption~\ref{ass:scalar}, for \emph{every} $V_T\ge0$,
\[R(T)\le L \sum_{t=1}^T\E\vert{}\bar q_t-u_t^*\vert{} =O\!\left(f_{\min}^{-1}T^{5/6}(1+V_T)^{1/6} +f_{\min}^{-1/2}T^{5/6}(1+V_T)^{-1/3}\right).\]
If in addition $f_t(q)\le f_{\max}$ for all $t,q\in A$ (a two-sided version of Assumption A2), then
\[Q(T)=\sum_{t=1}^T\E\vert{}F_t(\bar q_t)-(1-\alpha)\vert{} =O\!\left(\frac{f_{\max}}{f_{\min}}T^{5/6}(1+V_T)^{1/6} +\frac{f_{\max}}{\sqrt{f_{\min}}}T^{5/6}(1+V_T)^{-1/3}\right).\]
\end{theorem}

\begin{proof}
{\bf Proof strategy.}
We first identify a specific grid point $W_0 \in \mathcal{W}$ that closely approximates the optimal continuous window $W^\dagger$ derived in Theorem~\ref{thm:scalar-upper}. We then show the adaptive (Hedge) algorithm matches this benchmark, translating between the value domain (pinball loss) and the point/gradient domains.

{\bf Part I: A uniform grid benchmark.}
Recall from Theorem~\ref{thm:scalar-upper} the optimal continuous window $W^\dagger=\left(\frac{C_{VC}T}{2(1+V_T)}\right)^{2/3}$, which achieves $Q_{W^\dagger}(T) = O\!\left(T^{2/3}(1+V_T)^{1/3}\right)$ for all $V_T \ge 0$. Because Hedge only chooses among the powers of two in $\mathcal{W}$, we must convert this continuous benchmark into one achieved by an actual grid point at the cost of only a constant factor.

Let $h(W)=\frac{C_{VC}T}{\sqrt W}+V_TW$.
\end{proof}
\begin{lemma}[Constant-factor grid approximation]
\label{lem:grid}
For any reference window $W_{\mathrm{ref}}>0$ and scalar $c\in[1/2,2]$, we have $h(cW_{\mathrm{ref}}) \le 3h(W_{\mathrm{ref}})$.
\end{lemma}
\begin{proof}
Let $h(W_{\mathrm{ref}})=a+b$, where $a=C_{VC}T/\sqrt{W_{\mathrm{ref}}}$ and $b=V_TW_{\mathrm{ref}}\ge0$. Then $h(cW_{\mathrm{ref}})=a/\sqrt{c}+cb \le (1/\sqrt{c}+c)h(W_{\mathrm{ref}})$, because the upper bound includes the additional nonnegative cross terms. The convex function $c\mapsto 1/\sqrt{c}+c$ attains its maximum on the interval $[1/2,2]$ at one of its endpoints. Since $f(2)=1/\sqrt{2}+2 < 3$ exceeds $f(1/2)=\sqrt{2}+1/2$, we conclude $h(cW_{\mathrm{ref}})\le 3h(W_{\mathrm{ref}})$.
\end{proof}

Let $W_0$ be the element in the grid $\mathcal{W}$ nearest to $W^\dagger$.
While Lemma~\ref{lem:grid} controls $h(W)$, the initialization term $W$ and the residual $T/W$ must be bounded separately to guarantee the total coverage violation $Q_{W_0}(T)$. We establish this by considering two cases based on the optimal continuous window $W^\dagger$:

\emph{Case 1: $W^\dagger\ge1$.}  By construction, $W_0/W^\dagger\in[1/2,2]$. Applying Lemma~\ref{lem:grid}, $h(W_0)\le 3h(W^\dagger)=O(T^{2/3}(1+V_T)^{1/3})$. Furthermore, the remaining terms scale by at most a factor of two: $W_0\le 2W^\dagger = O((T/(1+V_T))^{2/3})$ and $T/W_0\le 2T/W^\dagger = O(T^{2/3}(1+V_T)^{1/3})$. Summing these components gives:
\begin{equation}
\label{eq:QW0}
Q_{W_0}(T)\le W_0+h(W_0)+\frac T{W_0}
=O\!\left(\left(\frac{T}{1+V_T}\right)^{\!2/3}\right)+O\!\left(T^{2/3}(1+V_T)^{1/3}\right).
\end{equation}

\emph{Case 2: $W^\dagger<1$.} Here, we simply choose the smallest grid element, $W_0:=1$. The condition $W^\dagger<1$ implies $1+V_T>C_{VC}T/2$. Consequently, $(1+V_T)^{1/3}>(C_{VC}/2)^{1/3}T^{1/3}$, which means $T^{2/3}(1+V_T)^{1/3} > (C_{VC}/2)^{1/3}T$. Since we also know $V_T\le T$, it follows that $T^{2/3}(1+V_T)^{1/3}=\Theta(T)$. Because the per-round coverage error $\vert{}F_t(q)-(1-\alpha)\vert{}$ is trivially bounded by $1$, we always have $Q_{W_0}(T)\le T$. Therefore, in this regime, $Q_{W_0}(T)=O(T)=O\bigl(T^{2/3}(1+V_T)^{1/3}\bigr)$. 

In both cases, \eqref{eq:QW0} holds for every $V_T\ge0$, seamlessly matching the continuous benchmark rate up to a constant factor. This fixed-window grid point $W_0 \in \mathcal{W}$ is the benchmark Layer 2 will compete against.

{\bf Part II: Matching the benchmark adaptively.}

We now show that the Hedge algorithm, despite lacking knowledge of $V_T$, achieves the benchmark rate \eqref{eq:QW0} up to a square root. 

We now formalize the geometry of the population pinball loss underlying the translation between
value-domain (Hedge) regret and point-domain tracking error.

Let
\begin{equation}
\label{eq:Psi-def}
\Psi_t(q) := \E[\psi_t(q)\mid \mathcal H_{t-1}] = \E_{r_t\sim F_t}[\psi_t(q)],
\end{equation}
the conditional expected pinball loss at threshold $q$, where $\psi_t$ is the pinball loss of
Definition~\ref{defn:pinball} and the conditioning is on the (fixed, given $\mathcal H_{t-1}$)
distribution $F_t$ of $r_t$.

\begin{lemma}\label{lem:PsiMin}
$\Psi_t$ is minimized at $u_t^*$.
\end{lemma}
Proof is given at the end of this section.

\begin{lemma}
\label{lem:loss_props}
Under Assumption~\ref{ass:scalar}, for every $q\in A$: (i) $\Psi_t$ is convex; (ii) $0\le\Psi_t(q)-\Psi_t(u_t^*)\le|q-u_t^*|$; (iii) $\Psi_t(q)-\Psi_t(u_t^*)\ge\frac{f_{\min}}2(q-u_t^*)^2$.
\end{lemma}

\begin{proof}
(i) Fix a realized value $r_t\in A$. The pinball loss
\[
\psi_t(q) = (1-\alpha)(r_t-q)_+ + \alpha(q-r_t)_+
\]
is a nonnegative combination of the two functions $q\mapsto(r_t-q)_+$ and $q\mapsto(q-r_t)_+$,
each of which is a composition of the convex function $(\cdot)_+=\max(\cdot,0)$ with an affine
map of $q$, and is therefore itself convex. So $\psi_t(\cdot)$ is convex \emph{for every fixed realized $r_t$}. Since $\psi_t$ is bounded on the bounded domain $A$ (so all three
expectations below are finite), taking $\E_{r_t\sim F_t}[\,\cdot\,]$ on both sides, a monotone,
linear operator, preserves the inequality:
\[
\Psi_t\bigl(\lambda q_1+(1-\lambda)q_2\bigr)
= \E_{r_t}\Bigl[\psi_t\bigl(\lambda q_1+(1-\lambda)q_2\bigr)\Bigr]
\le \lambda\,\E_{r_t}[\psi_t(q_1)] + (1-\lambda)\,\E_{r_t}[\psi_t(q_2)]
= \lambda\Psi_t(q_1)+(1-\lambda)\Psi_t(q_2).
\]
As $q_1,q_2\in A$ and $\lambda\in[0,1]$ were arbitrary, this is precisely the definition of
convexity for $\Psi_t$. 

(ii) $\psi_t$ is $\max(\alpha,1-\alpha)$-Lipschitz, hence $1$-Lipschitz, so
$|\Psi_t(q)-\Psi_t(u_t^*)| \le \E[\,|\psi_t(q)-\psi_t(u_t^*)|\mid\mathcal H_{t-1}\,] \le |q-u_t^*|$;
nonnegativity holds since $u_t^*$ minimizes $\Psi_t$ (shown above).

(iii) $\Psi_t'(q) = F_t(q)-(1-\alpha)$, so for $q\ge u_t^*$, using $F_t(u_t^*)=1-\alpha$ and
$f_t\ge f_{\min}$,
\[
\Psi_t(q)-\Psi_t(u_t^*) = \int_{u_t^*}^{q}\bigl(F_t(x)-(1-\alpha)\bigr)\,dx
\ge \int_{u_t^*}^{q} f_{\min}(x-u_t^*)\,dx = \frac{f_{\min}}{2}(q-u_t^*)^2;
\]
symmetrically for $q<u_t^*$.
\end{proof}

%

\emph{Step 1:} By Lemma~\ref{lem:loss_props}(iii), we have
\begin{equation}
\label{eq:strongcvx}
|q-u_t^*|\le\sqrt{\tfrac2{f_{\min}}\bigl(\Psi_t(q)-\Psi_t(u_t^*)\bigr)}
\qquad\text{for every }q\in A.
\end{equation}

\emph{Step 2:} Let $q_t^{(W_0)}$ be the threshold chosen by expert $W_0$. Then using Lemma~\ref{lem:loss_props}(ii), we have \[\Psi_{W_0}(T):=\sum_{t=1}^T\E\bigl[\Psi_t(q_t^{(W_0)})-\Psi_t(u_t^*)\bigr] \le\sum_{t=1}^T\E\bigl\vert{}q_t^{(W_0)}-u_t^*\bigr\vert{}=R_{W_0}(T) \le W_0+\frac1{f_{\min}}\bigl(h(W_0)+T/W_0\bigr).\]
Substituting $W_0=O((T/(1+V_T))^{2/3})$ and $h(W_0)+T/W_0=O(T^{2/3}(1+V_T)^{1/3})$  yields:
\begin{equation}
\label{eq:RW0}
\Psi_{W_0}(T)\le R_{W_0}(T)
=O\!\left(\left(\frac{T}{1+V_T}\right)^{\!2/3}\right)
+O\!\left(\frac{T^{2/3}(1+V_T)^{1/3}}{f_{\min}}\right).
\end{equation}

\emph{Step 3 (Hedge regret):} Because $\psi_t(q)\in[0,M]$ where $M:=\max(\alpha,1-\alpha)$, the standard full-information Hedge bound \cite{FreundSchapire1997} guarantees that the mixture's loss stays close to the best grid point's loss. Deterministically, for every realization of $\{r_t\}$:
\begin{equation}
\label{eq:hedgebound}
\sum_{t=1}^T\psi_t(\bar q_t)\le\sum_{t=1}^T\sum_{W\in\mathcal{W}}
w_{t,W}\psi_t(q_t^{(W)})
\le\sum_{t=1}^T\psi_t(q_t^{(W_0)})+O\bigl(M\sqrt{T\log\log T}\bigr).
\end{equation}
Here, the first inequality applies Jensen's inequality to the convex $\psi_t$, and the second uses the Hedge guarantee \cite{FreundSchapire1997} with $\ln\vert{}\mathcal{W}\vert{}=O(\log\log T)$. Taking conditional expectations and applying Lemma~\ref{lem:loss_props}(i) gives:
\begin{equation}
\label{eq:psimix}
\Psi_{\mathrm{mix}}(T):=\sum_{t=1}^T\E\bigl[\Psi_t(\bar q_t)-\Psi_t(u_t^*)\bigr]
\le\Psi_{W_0}(T)+O\bigl(\sqrt{T\log\log T}\bigr).
\end{equation}

\emph{Step 4 (Translating value regret to tracking error):} 

We now return to  bound $R(T)$ using \eqref{eq:strongcvx}. By Jensen's
inequality ($\E[\sqrt X]\le\sqrt{\E[X]}$) and Cauchy--Schwarz over $t$, we obtain:
\[
R(T)\le L \sum_{t=1}^T\E|\bar q_t-u_t^*|
\le L \sqrt{T\sum_{t=1}^T\tfrac2{f_{\min}}\E\bigl[\Psi_t(\bar q_t)-\Psi_t(u_t^*)\bigr]}
=L \sqrt{\tfrac{2T}{f_{\min}}\,\Psi_{\mathrm{mix}}(T)}.
\]

Substituting the bound on $\Psi_{\mathrm{mix}}(T)$ from \eqref{eq:psimix}, which combined
\eqref{eq:RW0} with the Hedge regret bound \eqref{eq:hedgebound},
\[
\Psi_{\mathrm{mix}}(T) = O\!\left(\left(\frac{T}{1+V_T}\right)^{2/3}
+ \frac{T^{2/3}(1+V_T)^{1/3}}{f_{\min}} + \sqrt{T\log\log T}\right),
\]
so that
\[
R(T) = O\!\left(\sqrt{\frac{T}{f_{\min}}}\,
\sqrt{\left(\frac{T}{1+V_T}\right)^{2/3}
+ \frac{T^{2/3}(1+V_T)^{1/3}}{f_{\min}} + \sqrt{T\log\log T}}\right).
\]

Using $\sqrt{a+b+c}\le\sqrt a+\sqrt b+\sqrt c$ and distributing the outer factor
$\sqrt{T/f_{\min}}$ across the three terms:
\begin{align*}
\text{Term 1:}\quad
&\sqrt{\frac{T}{f_{\min}}}\left(\frac{T}{1+V_T}\right)^{1/3}
= f_{\min}^{-1/2}\,T^{5/6}(1+V_T)^{-1/3},\\[4pt]
\text{Term 2:}\quad
&\sqrt{\frac{T}{f_{\min}}}\cdot\frac{T^{1/3}(1+V_T)^{1/6}}{\sqrt{f_{\min}}}
= f_{\min}^{-1}\,T^{5/6}(1+V_T)^{1/6},\\[4pt]
\text{Term 3:}\quad
&\sqrt{\frac{T}{f_{\min}}}\,(T\log\log T)^{1/4}
= f_{\min}^{-1/2}\,T^{3/4}(\log\log T)^{1/4} = o\!\left(T^{5/6}\right).
\end{align*}
Term 3 is strictly lower order than Terms 1--2 and is absorbed into the $O(\cdot)$.
Collecting Terms 1--2 gives
\begin{equation}
R(T) = O\!\left(f_{\min}^{-1}\,T^{5/6}(1+V_T)^{1/6}
+ f_{\min}^{-1/2}\,T^{5/6}(1+V_T)^{-1/3}\right). \label{eq:RTfinal}
\end{equation}

\emph{Step 5 (Coverage violation $Q(T)$).} It remains to bound
$Q(T)=\sum_{t=1}^T\E|F_t(\bar q_t)-(1-\alpha)|$. Here we invoke the two-sided density
assumption $f_t(q)\le f_{\max}$ for all $t,q\in A$, so that $F_t$ is globally
$f_{\max}$-Lipschitz. Since $F_t(u_t^*)=1-\alpha$ by definition of the oracle quantile,
\[
|F_t(\bar q_t)-(1-\alpha)| = |F_t(\bar q_t)-F_t(u_t^*)| \le f_{\max}\,|\bar q_t - u_t^*|
\]
deterministically, for every $t$. Taking expectations and summing over $t=1,\dots,T$,
\[
Q(T) = \sum_{t=1}^T \E|F_t(\bar q_t)-(1-\alpha)|
\;\le\; f_{\max}\sum_{t=1}^T \E|\bar q_t - u_t^*|
\;=\; f_{\max}\,R(T).
\]
Substituting \eqref{eq:RTfinal} gives
\[
Q(T) = O\!\left(\frac{f_{\max}}{f_{\min}}\,T^{5/6}(1+V_T)^{1/6}
+ \frac{f_{\max}}{\sqrt{f_{\min}}}\,T^{5/6}(1+V_T)^{-1/3}\right),
\]

Next, we present the remaining proof for Lemma \ref{lem:PsiMin}.
\begin{proof} (Proof of Lemma \ref{lem:PsiMin}.)
Since $\psi_t(\cdot)$ is convex (shown below in Lemma~\ref{lem:loss_props}(i)) and piecewise
linear with subgradient uniformly bounded by $\max(\alpha,1-\alpha)\le 1$ on the bounded domain
$A$, dominated convergence lets us differentiate \eqref{eq:Psi-def} under the expectation:
\begin{align*}
\Psi_t'(q) &\;=\; \E_{r_t\sim F_t}\bigl[\partial_q\psi_t(q)\bigr]
\;=\; -(1-\alpha)\Pr(r_t> q) + \alpha\Pr(r_t\le q), \\
&\;=\; -(1-\alpha)(1-F_t(q)) + \alpha F_t(q)
\;=\; F_t(q)-(1-\alpha),
\end{align*}
using the pinball subdifferential ($\partial\psi_t(q)=-(1-\alpha)$ for $q<r_t$ and $\partial\psi_t(q)=\alpha$
for $q>r_t$, from Definition~\ref{defn:pinball}) and averaging it over $r_t\sim F_t$. Because $F_t$
is nondecreasing, $\Psi_t'$ is nondecreasing (consistent with the convexity established in part
(i) below), and
\[
\Psi_t'(q) \;<0 \ \text{ for } q<u_t^*, \qquad \Psi_t'(u_t^*)=0, \qquad \Psi_t'(q) \;>0 \ \text{ for } q>u_t^*,
\]
since $u_t^*=F_t^{-1}(1-\alpha)$ is exactly the point where $F_t(q)=1-\alpha$. Hence $\Psi_t$ is minimized at $u_t^*$.
\end{proof}
\begin{remark}
Without the knowledge of $V_T$, the degradation in regret and coverage violation from $O(T^{2/3})$ to $O(T^{5/6})$ occurs because Hedge controls the \emph{value domain} (pinball loss regret) rather than the \emph{point domain} (tracking error $\vert{}q-u_t^*\vert{}$). Translating the value regret back into the point domain relies on the loss's quadratic margin, which introduces a square root. Summing these errors over $T$ steps via Cauchy--Schwarz mechanically inflates the bound by a factor of $\sqrt{T}$, yielding $\sqrt{T \cdot T^{2/3}} = T^{5/6}$.
\end{remark}

%
\section{Proof of Theorem \ref{thm:scalar-lower}}

\begin{proof}
Since we want to prove a  \emph{lower bound} over the class of monotone $L$-Lipschitz function $\Lambda$, thus, we let
\[
\Lambda(q) := q, \qquad \text{for every } t,
\]
which is monotone nondecreasing and $1$-Lipschitz, hence a valid instance with
$L=1$. With this choice,
\[
\Lambda(q_t)-\Lambda(u_t^*) = q_t - u_t^*,
\]
so $R(T) = \sum_{t=1}^T \E\bigl|\Lambda(q_t)-\Lambda(u_t^*)\bigr| = \sum_{t=1}^T \E\big|q_t-u_t^*\bigr|$,
and it suffices to lower-bound $\sum_{t=1}^T \E|q_t-u_t^*|$. 

The proof proceeds in four steps: (1) construct two
hard distributions, (2) bound the sample complexity of distinguishing them,
(3) embed the testing problem into the tracking problem via a block structure,
and (4) verify the resulting distribution sequence is within $\mathcal{F}(V_T)$.

\textbf{Step 1: Two-distribution construction.}
Fix $\alpha\in(0,1)$ and a small perturbation $\varepsilon\in(0,\alpha)$.
Let $P_0$ be uniform on $[0,1]$ (density $f_0\equiv 1$, oracle quantile
$u_0^*=1-\alpha$). Let $P_1$ have piecewise-constant density
\[
f_1(q)=\begin{cases}1+\varepsilon&q\in[0,1-\alpha],\\
1-\frac{(1-\alpha)\varepsilon}{\alpha}&q\in(1-\alpha,1].\end{cases}
\]
Note $f_1$ integrates to $1$ and is a valid density for $\varepsilon<\alpha$.
The oracle quantile for $P_1$ solves $(1+\varepsilon)u_1^*=1-\alpha$,
giving $u_1^*=\frac{1-\alpha}{1+\varepsilon}$.
The separation is $\Delta:=u_0^*-u_1^*=\frac{(1-\alpha)\varepsilon}{1+\varepsilon}
=\Theta(\varepsilon)$.

The KL divergence (Taylor expanding to second order in $\varepsilon$):
\[
\dkl{P_0}{P_1}=(1-\alpha)\log\frac{1}{1+\varepsilon}+\alpha\log\frac{1}{1-\frac{(1-\alpha)\varepsilon}{\alpha}}
=\frac{(1-\alpha)\varepsilon^2}{2\alpha}+O(\varepsilon^3)\le C\Delta^2,
\]
for a constant $C=\frac{1-\alpha}{2\alpha(1-\alpha)^2}\cdot(1+o(1))$
(using $\varepsilon=\Theta(\Delta)$).

\textbf{Step 2: Sample complexity of testing.}
Divide the horizon $T$ into $K$ blocks of length $B=\lfloor T/K\rfloor$.
Draw labels $\Theta=(\Theta_1,\dots,\Theta_K)$, each $\Theta_k\in\{0,1\}$
i.i.d.\ uniform (unknown to the learner), and generate all $B$ scores in
block $k$ i.i.d.\ from $P_{\Theta_k}$. For a \emph{fixed} label sequence
$\theta\in\{0,1\}^K$, write $\{F_t^\theta\}$ for the resulting
deterministic CDF sequence, with corresponding oracle sequence
$u_t^{*,\theta}=u^*_{\theta_k}$ for $t$ in block $k$. Define the
\textbf{Bayes risk}
\[
\bar R(T;\mathcal A):=\E_\Theta\bigl[R(T;\mathcal A,\{F_t^\Theta\})\bigr].
\]

Consider any round $t$ in block $k$. The learner's threshold
$q_t$ depends on all past scores $(r_1,\dots,r_{t-1})$ and
internal randomness $U$.
Let $\mathcal{H}_{k-1}$ be the $\sigma$-algebra of all information
before block $k$ begins. By construction, $\Theta_k$ is independent
of $\mathcal{H}_{k-1}$ and $U$.

\emph{The decoder.}
We use $q_t$ to construct a guess for $\Theta_k$. Define:
\[
\phi_t:=\mathbf{1}\!\left\{|q_t-u_1^*|<|q_t-u_0^*|\right\}.
\]
In words: $\phi_t=1$ if the learner's threshold is closer to $u_1^*$
than to $u_0^*$, and $\phi_t=0$ otherwise. This is the natural
``nearest-neighbour'' decoder for the two-point hypothesis test.
If $\phi_t\ne\Theta_k$, then $q_t$ is on the wrong side of the midpoint
$(u_0^*+u_1^*)/2$, which implies
\[
|q_t-u_{\Theta_k}^*|\ge\frac{\Delta}{2}.
\]
(When $\Theta_k=0$ and $\phi_t=1$: $q_t$ is closer to $u_1^*$, so
$|q_t-u_0^*|>|q_t-u_1^*|\ge 0$, and by the midpoint argument
$|q_t-u_0^*|\ge\Delta/2$. Symmetric for $\Theta_k=1$.)

\emph{Bounding the error probability.}
Conditional on $(\mathcal{H}_{k-1},U)=(h,u)$, the decoder $\phi_t$ is
a deterministic measurable function of the current block's $B$ scores,
which are i.i.d.\ from $P_{\Theta_k}$. Le Cam's lemma \cite{Tsybakov2009} applied to this
conditional binary testing problem gives:
\[
\Prob(\phi_t\ne\Theta_k\mid\mathcal{H}_{k-1}=h,U=u)
\ge\frac{1}{2}\bigl(1-\tv{P_0^B,P_1^B}\bigr).
\]
where $\text{TV}$ is the total variation distance.
By Pinsker's inequality and tensorization of KL divergence:
\[
\tv{P_0^B,P_1^B}
\le\sqrt{\tfrac{1}{2}\dkl{P_0^B}{P_1^B}}
=\sqrt{\tfrac{B}{2}\dkl{P_0}{P_1}}
\le\sqrt{\tfrac{BC\Delta^2}{2}}.
\]
Choosing $B=\lfloor\frac{1}{8C\Delta^2}\rfloor$ ensures
$\tv{P_0^B,P_1^B}\le\frac{1}{4}$, so $\Prob(\phi_t\ne\Theta_k\mid\mathcal{H}_{k-1},U)\ge\frac{3}{8}$.
Taking expectation over $(\mathcal{H}_{k-1},U,\Theta_k)$:
$$\Prob(\phi_t\ne\Theta_k)\ge\frac{3}{8}.$$

\textbf{Step 3: Lower bounding the Bayes risk.}
Fix a block $k$. Combining $|q_t-u^*_{\Theta_k}|\ge\frac{\Delta}{2}\mathbf 1\{\phi_t\ne\Theta_k\}$
with $\Prob(\phi_t\ne\Theta_k)\ge\frac38$ from Step~2, and summing over
the $B$ rounds in block $k$:
\[
\sum_{t\in\text{block }k}\E|q_t-u_{\Theta_k}^*|
\ge\frac{\Delta}{2}\sum_{t\in\text{block }k}\Prob(\phi_t\ne\Theta_k)
\ge\frac{3\Delta B}{8}.
\]
By definition of the Bayes risk and linearity of expectation, together
with $\Lambda(q)=q$ so that the tracking error at each round $t$
coincides with the $\Lambda$-regret measured against $u_t^{*,\Theta}$,
which equals $u^*_{\Theta_k}$ for all $t$ in block $k$; hence
\[
\bar R(T;\mathcal A)
=\E_\Theta\Bigl[\sum_{t=1}^T\E\bigl|q_t-u_{t\Theta}^*\bigr|\Bigr]
=\E_\Theta\Bigl[\sum_{k=1}^K\sum_{t\in\text{block }k}\E\bigl|q_t-u^*_{\Theta_k}\bigr|\Bigr]
=\sum_{k=1}^K\E_{\Theta_k}\Bigl[\sum_{t\in\text{block }k}\E\bigl|q_t-u^*_{\Theta_k}\bigr|\Bigr],
\]
where the last equality holds because each inner sum depends on $\Theta$
only through $\Theta_k$. Applying the per-block bound established above
to each of the $K$ terms,
\[
\bar R(T;\mathcal A)\ge K\cdot\frac{3\Delta B}{8}
=\frac{3\Delta}{8}\cdot KB=\frac{3T\Delta}{8}=:c_3T\Delta,
\]
using $KB=T$ (up to the floor in $B=\lfloor T/K\rfloor$, which
contributes only a lower-order correction absorbed into $c_3$).

\emph{From the Bayes risk to the supremum.} By definition,
$\bar R(T;\mathcal A)$ is the average of $R(T;\mathcal A,\{F_t^\theta\})$
over $\theta\in\{0,1\}^K$ under the uniform prior, and each
$R(T;\mathcal A,\{F_t^\theta\})$ is itself already an expectation over
the data alone. Since the average of finitely many numbers never
exceeds their maximum,
\[
\max_{\theta\in\{0,1\}^K}R(T;\mathcal A,\{F_t^\theta\})
\ge\bar R(T;\mathcal A)\ge c_3T\Delta.
\]
Fix any $\theta^{\mathcal A}$ attaining this maximum over the finite set
$\{0,1\}^K$ and write $\{F_t^{\mathcal A}\}:=\{F_t^{\theta^{\mathcal A}}\}$
for the corresponding (now fixed, deterministic) CDF sequence. Then
\begin{equation}\label{eq:finallb}
R\bigl(T;\mathcal A,\{F_t^{\mathcal A}\}\bigr)\ge c_3T\Delta.
\end{equation}

\textbf{Step 4: Verifying the variation budget.}
The sequence $\{F_t^{\mathcal A}\}$ changes CDF only at block
boundaries where $\theta^{\mathcal A}_k$ changes. At each such switch,
$\sup_q|F_0(q)-F_1(q)|=(1-\alpha)\varepsilon=(1+\varepsilon)\Delta$.
The number of switches is at most $K$. Therefore:
\[
V_{\rm seq}\le K(1+\varepsilon)\Delta=\frac{T}{B}(1+\varepsilon)\Delta.
\]
For sufficiently small $\Delta$, $B\ge\frac{1}{16C\Delta^2}$
(using the floor bound $B\ge\frac{1}{8C\Delta^2}-1$), so:
\[
V_{\rm seq}\le 16CT\Delta^3(1+\varepsilon)\le C'T\Delta^3,
\]
where $C'>0$ absorbs the $(1+\varepsilon)$ factor (valid for $\varepsilon$
small, i.e.\ $\Delta$ small).

\emph{Choosing $\Delta$ to match $V_T$.}
Set $\Delta=c_0(V_T/T)^{1/3}$ for small $c_0>0$.
Then $V_{\rm seq}\le C'c_0^3V_T\le V_T$ for $c_0\le(1/C')^{1/3}$, so
$\{F_t^{\mathcal A}\}\in\mathcal D(V_T)$. Using this choice of $\Delta$ in \eqref{eq:finallb}, we get
\[
R\bigl(T;\mathcal A,\{F_t^{\mathcal A}\}\bigr)\ge c_3T\Delta
=c_3c_0T^{2/3}V_T^{1/3}=:c_1T^{2/3}V_T^{1/3}.
\]
This holds for $V_T\le c_2T$ where $c_2>0$ ensures $\varepsilon<\alpha$
(a constraint from the density construction). Since
$\{F_t^{\mathcal A}\}\in\mathcal D(V_T)$, this gives
$$\sup_{\{F_t\}\in\mathcal D(V_T)}R(T;\mathcal A,\{F_t\})
\ge R(T;\mathcal A,\{F_t^{\mathcal A}\})\ge c_1T^{2/3}V_T^{1/3},$$ and
since $\mathcal A$ was arbitrary, this holds for every algorithm,
proving the bound on $R(T)$ (with $\Lambda(q)=q$ as the admissible
witness within the monotone $L$-Lipschitz class, which suffices for a
lower bound over that class).

\emph{Bound for $Q(T)$.}
The geometric coupling (MVT, density bounded below by $f_{\min}$):
$|F_t(q_t)-(1-\alpha)|\ge f_{\min}|q_t-u_t^*|$, applied at
$\{F_t^{\mathcal A}\}$ and summed, gives
$$Q(T;\mathcal A,\{F_t^{\mathcal A}\})\ge f_{\min}R(T;\mathcal A,\{F_t^{\mathcal A}\})
\ge c_1f_{\min}T^{2/3}V_T^{1/3},$$ and the same supremum-and-arbitrary-$\mathcal A$
argument as above completes the proof.
\end{proof}

\section{Proof of Theorem~\ref{thm:conditional}}

\subsection{Step 1: Cell indexing and the high-probability visit event}

Fix a cell $j$ and relabel its visit rounds by $k=1,\ldots,m$, where $m=m_j$ is the (random)
number of times cell $j$ is visited over the horizon $T$. Throughout, $t_k$ denotes the time
of the $k$-th visit to cell $j$, and \emph{every quantity subsequently subscripted by $k$ is
shorthand for that same quantity evaluated at the (random) time $t_k$}: thus
$u_k^*:=u_{t_k}^*(x_j)$, $q_k$ is the algorithm's threshold in cell $j$ immediately after the
$(k-1)$-th update (i.e.\ at time $t_k$), $e_k:=q_k-u_k^*$, $h_k(\cdot)$ is the function
$h_{t_k}(\cdot\mid x_j)$ defined below, and similarly for $\mathrm{noise}_k$, $\mathrm{bias}_k$,
and $\Delta_k$ introduced in Step 2.

Because the argument below conditions on every cell receiving enough visits, we isolate this
fact first.

\begin{lemma}[Visit concentration]
\label{lem:visits}
With probability at least $1-h^{-d}\exp(-p_{\min}h^dT/8)$, every cell satisfies $m_j \ge \frac{1}{2}p_{\min}h^dT$.
\end{lemma}

\begin{proof}
By Assumption~C4, each indicator $\mathbf{1}(X_t \in B_j)$ is an independent Bernoulli random variable with mean $\mu_j \ge p_{\min}h^d$. Chernoff's lower-tail inequality gives $\Pr(m_j \le \mu_j/2) \le e^{-\mu_j/8}$. A union bound over the $N \le h^{-d}$ cells completes the proof.
\end{proof}

Under Assumption~C3, $h^dT \rightarrow \infty$ at the optimal bandwidth $h=h^*$, so the failure
probability $\delta_T:=h^{-d}\exp(-p_{\min}h^dT/8)$ in Lemma~\ref{lem:visits} is
super-polynomially small. Let $\mathcal{E}$ denote the event that \emph{every} cell satisfies
$m_j \in \left[\frac{1}{2}p_{\min}h^dT,\, p_{\max}h^dT\right]$ (the upper bound is the
high-probability Chernoff \emph{upper}-tail counterpart used identically to Lemma~\ref{lem:visits},
combined by the same union bound); thus $\Pr(\mathcal{E}^c)\le \delta_T = o(1)$.

Steps 2--5 below bound $\E[\,\cdot\mid \mathcal{E}\,]$, i.e.\ they are carried out
\emph{conditionally on $\mathcal{E}$}. This is sufficient for the unconditional in-expectation
statement of the theorem: for any of the nonnegative error functionals $X\in\{Q(T),R(T)\}$
considered here, $X$ is deterministically bounded by $X\le X_{\max}=O(T)$ (thresholds lie in the
bounded set $A$), so by the law of total expectation,
\[
\E[X] \;=\; \E[X\mid\mathcal{E}]\Pr(\mathcal{E}) \;+\; \E[X\mid\mathcal{E}^c]\Pr(\mathcal{E}^c)
\;\le\; \E[X\mid\mathcal{E}] \;+\; X_{\max}\,\delta_T
\;=\; \E[X\mid\mathcal{E}] + o(1),
\]
since $\delta_T$ is super-polynomially small in $T$ while $X_{\max}=O(T)$ is only polynomial.
Hence bounding $\E[X\mid\mathcal E]$, as Steps 2--5 do, controls $\E[X]$ up to an additive $o(1)$
term, which we henceforth suppress. We condition on $\mathcal E$ for the remainder of the proof.

\subsection{Step 2: Per-cell error recursion}

Define $h_k(q) = (1-F_{t_k}(q\mid x_j))-\alpha$. Since $F_{t_k}(u_k^*\mid x_j)=1-\alpha$, the
function $h_k$ has a root at $u_k^*$, with derivative
\begin{equation}
\label{eq:hk-deriv}
h_k'(q) = -f_{t_k}(q\mid x_j) \le -f_{\min}^s < 0,
\end{equation}
where $f_{t_k}(\cdot\mid x_j)$ denotes the conditional conformity-score density at time $t_k$,
and the inequality is Assumption~C4's density lower bound.

Before decomposing the update signal, we isolate its two error sources. The
martingale-difference \emph{noise} term is
\[
\mathrm{noise}_k \;:=\; (1-z_k) - \bigl(1-F_{t_k}(q_k\mid X_{t_k})\bigr),
\]
which satisfies $\E[\mathrm{noise}_k\mid\mathcal{H}_{k-1},X_{t_k}]=0$ and $|\mathrm{noise}_k| \le 2$
(it is the difference of two $[0,1]$-valued quantities). The spatial \emph{bias} term is
\[
\mathrm{bias}_k \;:=\; F_{t_k}(q_k\mid x_j) - F_{t_k}(q_k\mid X_{t_k}),
\]
which, by Assumption~C2 and since $X_{t_k}\in B_j$ lies within a cell of side length $h$ (so
$\|x_j-X_{t_k}\|\le \sqrt{d}\,h$), is bounded by $|\mathrm{bias}_k| \le L_xd^{\beta/2}h^\beta$.

Recall the HP-ACI update rule~(12), $q^{j(t)}_{t+1}=\Pi_A\bigl[q^{j(t)}_t+\eta((1-z_t)-\alpha)\bigr]$;
restricted to the visit rounds of cell $j$, this update decomposes the update signal as
\begin{equation}\label{eq:decomp}
(1-z_k)-\alpha = h_k(q_k) + \mathrm{noise}_k + \mathrm{bias}_k.
\end{equation}

By the Mean Value Theorem, because $h_k'(q)=-f_{t_k}(q\mid x_j)$ as computed in
\eqref{eq:hk-deriv}, there exists $\xi_k$ between $q_k$ and $u_k^*$ such that
\[
h_k(q_k) \;=\; h_k(q_k)-h_k(u_k^*) \;=\; h_k'(\xi_k)\,(q_k-u_k^*) \;=\; -f_{t_k}(\xi_k\mid x_j)\,e_k.
\]
Because projection is non-expansive and $u_{k+1}^*\in A$, applying \eqref{eq:decomp} to the
update rule yields:
\begin{equation}
\label{eq:cond-recur}
|e_{k+1}| \le |e_k(1-\eta f_{t_k}(\xi_k\mid x_j))| + \eta|\mathrm{noise}_k| + \eta|\mathrm{bias}_k| + |\Delta_k|,
\end{equation}
where $\Delta_k = u_{k+1}^*-u_k^*$ is the oracle drift between consecutive visits.

\subsection{Step 3: Lyapunov recursion}

Define $v_k=\E[e_k^2]$ and $\rho_k = 1-\eta f_{t_k}(\xi_k\mid x_j)$. Under Assumption~C5 ($\eta \le f_{\min}^s / (2(f_{\max}^s)^2)$), we verify that $\rho_k \in [0,1)$ because $\eta f_{t_k}(\xi_k\mid x_j) \le \eta f_{\max}^s \le f_{\min}^s / (2f_{\max}^s) \le 1$.

\paragraph{Squaring and eliminating the noise cross-terms (in full).}
The bound \eqref{eq:cond-recur} arises from the projection non-expansiveness applied to the
\emph{signed} pre-projection quantity: writing $$\hat e_{k+1} := e_k\rho_k + \eta\,\mathrm{noise}_k
+ \eta\,\mathrm{bias}_k - \Delta_k,$$ non-expansiveness gives $|e_{k+1}|\le|\hat e_{k+1}|$, hence
$e_{k+1}^2\le \hat e_{k+1}^2$. 

Let $\mathcal H_k:=\cF(\mathcal H_{k-1},X_{t_k},q_k)$ be the
information available immediately before $z_k$ is revealed. Note that $e_k,\rho_k,\mathrm{bias}_k,\Delta_k$
are all $\mathcal H_k$-measurable (or deterministically bounded given $\mathcal H_k$), while
$\E[\mathrm{noise}_k\mid\mathcal H_k]=0$. Expanding the square,
\[
\hat e_{k+1}^2 = \rho_k^2e_k^2 + \eta^2\mathrm{noise}_k^2 + (\eta\,\mathrm{bias}_k-\Delta_k)^2
+ 2\rho_ke_k\,\eta\,\mathrm{noise}_k + 2\rho_ke_k(\eta\,\mathrm{bias}_k-\Delta_k)
+ 2\eta\,\mathrm{noise}_k(\eta\,\mathrm{bias}_k-\Delta_k),
\]
and taking $\E[\,\cdot\mid\mathcal H_k]$, the two cross-terms linear in $\mathrm{noise}_k$ vanish
because they factor as $(\mathcal H_k\text{-measurable})\times\E[\mathrm{noise}_k\mid\mathcal H_k]=0$.
Using $e_{k+1}^2\le \hat e_{k+1}^2$, we get 
\[
\E[e_{k+1}^2\mid\mathcal H_k] \le \rho_k^2e_k^2 + \eta^2\sigma^2 + (\eta\,\mathrm{bias}_k-\Delta_k)^2
+ 2\rho_ke_k(\eta\,\mathrm{bias}_k-\Delta_k)\] where  $\sigma^2:=\E[\mathrm{noise}_k^2\mid\mathcal H_k]\le 4$.

Bounding $(\eta\,\mathrm{bias}_k-\Delta_k)^2\le Y^2$ and $2\rho_ke_k(\eta\,\mathrm{bias}_k-\Delta_k)
\le 2\rho_k|e_k|Y$ with $Y:=\eta|\mathrm{bias}_k|+|\Delta_k|$ (worst-case sign), and then applying
Young's inequality $2\rho_ke_kY \le \eta f_{\min}^s e_k^2 + \frac{\rho_k^2}{\eta f_{\min}^s}Y^2$, gives:
\begin{equation}\label{eq:condexpe_k}
\E[e_{k+1}^2\mid\mathcal{H}_k] \le (\rho_k^2+\eta f_{\min}^s)e_k^2 + \eta^2\sigma^2 + \left(1+\frac{\rho_k^2}{\eta f_{\min}^s}\right)(\eta|\mathrm{bias}_k|+|\Delta_k|)^2.
\end{equation}

\paragraph{Contracting the $e_k^2$ coefficient.} Under Assumption~C5, the combined coefficient
of $e_k^2$ satisfies
\[
\rho_k^2+\eta f_{\min}^s \le (1-\eta f_{\min}^s)^2+\eta f_{\min}^s
= 1-\eta f_{\min}^s\bigl(1-\eta f_{\min}^s\bigr) \le 1-\frac{1}{2}\eta f_{\min}^s.
\]
The first inequality holds because $f_{t_k}(\xi_k\mid x_j)\ge f_{\min}^s$ (Assumption~C4) implies
$\rho_k = 1-\eta f_{t_k}(\xi_k\mid x_j) \le 1-\eta f_{\min}^s$; both $\rho_k$ and $1-\eta f_{\min}^s$
are nonnegative (the latter since $\eta f_{\min}^s\le\eta f_{\max}^s\le 1/2<1$ by C5), and squaring
is order-preserving on nonnegative reals, so $\rho_k^2\le(1-\eta f_{\min}^s)^2$; adding
$\eta f_{\min}^s$ to both sides gives the displayed inequality. The last inequality again uses
$\eta f_{\min}^s \le \eta f_{\max}^s \le 1/2$ (established above), so $1-\eta f_{\min}^s \ge 1/2$.

Since $1+\frac{\rho_k^2}{\eta p^s_{\min}} \;\le\; \frac{2}{\eta p^s_{\min}}$ (using $\eta p^s_{\min}\le 1$ by C5), setting
\begin{equation}
\label{eq:gamma-def}
\gamma:=\tfrac{1}{2}\eta f_{\min}^s,
\end{equation}
which is precisely the effective per-step contraction (geometric decay) rate of the recursion
below, and which satisfies $\gamma\in(0,1/2)$ since $\eta f_{\min}^s\le\eta f_{\max}^s\le 1/2$,
matching the hypothesis of Lemma~\ref{lem:telescoping}, and taking the unconditional expectation
of \eqref{eq:condexpe_k} yields the Lyapunov recurrence:
\begin{equation}
\label{eq:cond-lyap}
v_{k+1} \le (1-2\gamma)v_k + 4\eta^2 + \frac{4\eta L_x^2d^\beta h^{2\beta}}{f_{\min}^s} + \frac{4\Delta_k^2}{\eta f_{\min}^s}.
\end{equation}

\subsection{Step 4: Unrolling and step-size optimization}

Applying Lemma~\ref{lem:telescoping} to \eqref{eq:cond-lyap} bounds the cumulative variance:
\begin{equation}
\label{eq:cum-var}
\sum_{k=1}^m v_k = O\!\left( \frac{m\eta}{f_{\min}^s} + \frac{mh^{2\beta}}{(f_{\min}^s)^2} + \frac{(\mathcal{S}_T^j)^2}{\eta^2(f_{\min}^s)^2} \right),
\end{equation}
where $\mathcal{S}_T^j := \sum_{k=1}^{m-1} |\Delta_k| \le \mathcal{S}_T$ is the oracle path length within cell $j$, and we used $\sum_k\Delta_k^2 \le (\mathcal{S}_T^j)^2$.

\begin{remark}The shorthand $\gamma$ introduced in \eqref{eq:gamma-def} does not appear in \eqref{eq:cum-var}:
Lemma~\ref{lem:telescoping} produces the prefactor $1/(2\gamma)$, and substituting
$\gamma=\tfrac12\eta f_{\min}^s$ gives $1/(2\gamma)=1/(\eta f_{\min}^s)$, which is exactly the
factor multiplying each term of $\eqref{eq:cond-lyap}$'s additive part in \eqref{eq:cum-var}
(e.g.\ $4\eta^2/(\eta f_{\min}^s)=4\eta/f_{\min}^s$). Thus $\gamma$ was only a proof-internal
bookkeeping symbol for stating the contraction rate compactly in \eqref{eq:cond-lyap} and
Lemma~\ref{lem:telescoping}; once telescoped, it is fully absorbed into the explicit constants
$\eta$ and $f_{\min}^s$ and carries no separate meaning in the final rate.
\end{remark}

Using the high-probability bound $m_j \le m_{\max} := p_{\max}h^dT$ (valid on $\mathcal E$), we balance the first and third terms of \eqref{eq:cum-var} by choosing the global step size $\eta^\star = (2\mathcal{S}_T^2 / m_{\max}f_{\min}^s)^{1/3}$. Substituting $\eta^\star$ and keeping the dominant term yields the uniform estimate:
\[
\sum_{k=1}^{m}v_k = O\!\left( \frac{m_{\max}^{2/3}\mathcal{S}_T^{2/3}}{(f_{\min}^s)^{4/3}} + \frac{mh^{2\beta}}{(f_{\min}^s)^2} \right).
\]
Applying Cauchy--Schwarz and Jensen's inequality provides the point-error bound (stated per cell,
in terms of the local path length $\mathcal S_T^j$):
\begin{equation}
\label{eq:cell-bound}
\sum_{k=1}^{m}\E|e_k| \le \sqrt{m\sum_{k=1}^{m}v_k} = O\!\left( \frac{m^{1/2}m_{\max}^{1/3}(\mathcal{S}_T^j)^{1/3}}{(f_{\min}^s)^{2/3}} + \frac{mh^\beta}{f_{\min}^s} \right).
\end{equation}

\subsection{Step 5: Aggregation and balancing}

\paragraph{Aggregating across cells.}
Recall $N=\lceil h^{-d}\rceil$ is the number of axis-aligned cells $\{B_j\}_{j=1}^N$ covering
$\mathcal X=[0,1]^d$ (Section~4.3). Since $\mathcal{S}_T^j \le \mathcal{S}_T$ for all $j$, we have $\sum_j \mathcal{S}_T^j \le N\mathcal{S}_T \le h^{-d}\mathcal{S}_T$. By Holder's inequality, $\sum_j (\mathcal S_T^j)^{1/3} \le N^{2/3}(N\mathcal S_T)^{1/3} = N\,\mathcal S_T^{1/3} = h^{-d}\mathcal S_T^{1/3}$. On the high-probability event $m_j \le p_{\max}h^dT$, aggregating the first term of \eqref{eq:cell-bound} across cells — bounding $m_j^{1/2}\le m_{\max}^{1/2}$ pointwise so that $m_j^{1/2}m_{\max}^{1/3}\le m_{\max}^{5/6}$ — yields:
\[\sum_j m_j^{5/6}(\mathcal S_T^j)^{1/3} \le (p_{\max}h^dT)^{5/6}\, h^{-d}\,\mathcal S_T^{1/3} = O\!\left(T^{5/6}h^{-d/6}\mathcal S_T^{1/3}\right).\]

The cumulative spatial approximation error directly satisfies $\sum_j m_jh^\beta \le Th^\beta$, since $\sum_j m_j = T$ (every round visits exactly one cell).

\paragraph{Coverage violation.}
By the Mean Value Theorem, $(1-\alpha-F_t(q_t(X_t)\mid X_t))_+ \le f_{\max}^s|e_k| + f_{\max}^sLh^\beta$, where the second term captures the spatial approximation error between the oracle function and its cell representative. We now assemble the final bound explicitly. Summing the per-round violation first within a cell and then across all $N$ cells,
\[
Q(T) = \sum_{t=1}^T c_t(q_t) = \sum_{j=1}^N \sum_{k=1}^{m_j} \bigl(1-\alpha-F_{t_k}(q_k\mid x_j)\bigr)_+
\le f_{\max}^s\sum_{j=1}^N\sum_{k=1}^{m_j}|e_k| \;+\; f_{\max}^sL\sum_{j=1}^N m_jh^\beta.
\]
The second sum is $f_{\max}^sL\,Th^\beta = O(Th^\beta)$ by the display above. For the first sum,
apply the per-cell bound \eqref{eq:cell-bound} to each cell and sum:
\[
\sum_{j=1}^N\sum_{k=1}^{m_j}|e_k| = O\!\left(\frac{1}{(f_{\min}^s)^{2/3}}\sum_j m_j^{1/2}m_{\max}^{1/3}(\mathcal S_T^j)^{1/3} + \frac{h^\beta}{f_{\min}^s}\sum_j m_j\right)
= O\!\left(T^{5/6}h^{-d/6}\mathcal S_T^{1/3} + Th^\beta\right),
\]
using the two aggregated displays above (with the problem constants $f_{\min}^s,f_{\max}^s,L$
absorbed into the $O(\cdot)$). Multiplying by the outer factor $f_{\max}^s$ (a constant) and
combining with the second sum, both contributions are of the same two orders, giving
\[
Q(T) = O\!\left( T^{5/6} h^{-d/6} \mathcal{S}_T^{1/3} + Th^\beta \right).
\]

\paragraph{Balancing and Regret.}
To minimize the total error, we equate the bias and variance terms: $T^{5/6} h^{-d/6} \mathcal{S}_T^{1/3} \asymp Th^\beta$. Solving for $h$ yields $h^{\beta+d/6} \approx T^{-1/6} \mathcal{S}_T^{1/3}$, which gives the optimal bandwidth:
\[
h^* = T^{-1/(6\beta+d)} \mathcal{S}_T^{2/(6\beta+d)}.
\]
Substituting $h^*$ gives $T(h^*)^\beta = T^{\frac{5\beta+d}{6\beta+d}} \mathcal{S}_T^{\frac{2\beta}{6\beta+d}}$. Finally, by Assumption~C6,  $\Lambda(q_t(x),x)-\Lambda(u_t^*(x),x) \le L |q_t(x)-u_t^*(x)|$ and hence
\[
R(T) = O\!\left( \frac{L}{f_{\min}^s} T^{\frac{5\beta+d}{6\beta+d}} \mathcal{S}_T^{\frac{2\beta}{6\beta+d}} \right).
\]
\hfill$\blacksquare$

\begin{lemma}[Lyapunov telescoping]
\label{lem:telescoping}
If $v_{k+1} \le (1-2\gamma)v_k+C_k$ for $\gamma \in (0, 1/2)$, then
$\sum_{k=1}^m v_k \le \dfrac{v_1}{2\gamma} + \dfrac{1}{2\gamma} \sum_{k=1}^m C_k$.
\end{lemma}

\begin{proof}
Rearranging the recurrence gives $2\gamma v_k \le v_k - v_{k+1} + C_k$. Summing over $k=1, \ldots, m$ yields:
\[
2\gamma \sum_{k=1}^m v_k \le v_1 - v_{m+1} + \sum_{k=1}^m C_k \le v_1 + \sum_{k=1}^m C_k.
\]
Dividing by $2\gamma$ completes the proof.
\end{proof}

\end{document}